%% file: main.tex
\documentclass{article}

\usepackage{tikz} 
\usetikzlibrary{arrows.meta}
\usetikzlibrary{positioning}
\usetikzlibrary{decorations.pathmorphing}
\usetikzlibrary{decorations.pathreplacing}
\usepackage[utf8]{inputenc}
 \usepackage[preprint]{neurips_2025}
 \usepackage{tcolorbox}
 \usepackage{paralist}

\usepackage[utf8]{inputenc} %
\usepackage[T1]{fontenc}    %
\usepackage{url}            %
\usepackage{booktabs}       %
\usepackage{amsfonts}       %
\usepackage{nicefrac}       %
\usepackage{microtype}      %
\usepackage{xcolor}         %
\newcommand{\best}[1]{\makebox[0pt][l]{\textbf{\color{blue}#1}}\hphantom{#1}}
\newcommand{\secondbest}[1]{\makebox[0pt][l]{\underline{\color{green!50!black}#1}}\hphantom{#1}}
\newcommand{\pmv}[2]{#1$_{\scriptsize\pm\,#2}$}

\usepackage{algorithm}
\usepackage{algorithmic}
\usepackage{amssymb,amsmath,amsthm,enumitem}
\usepackage{amsfonts} 
\usepackage{bbm}

\usepackage{tikz}
\usetikzlibrary{arrows.meta}
\usetikzlibrary{bending}

\usepackage{hyperref}  
\usepackage{xspace}
\newcommand{\method}{{\sc PIQL}\xspace}

\input{dfn}

\usepackage{graphicx}
\usepackage{subcaption}
\usepackage[table]{xcolor}

\usepackage[capitalize,noabbrev]{cleveref}

\theoremstyle{plain}
\newtheorem{theorem}{Theorem}[section]

\theoremstyle{definition}

\theoremstyle{remark}

\title{Toward Privileged Foundation Models: \\LUPI for Accelerated and Improved Learning}

\author{%
  Xueying Ding \\
  Carnegie Mellon University\\
  \texttt{xding2@cs.cmu.edu} \\
  \And
  Leman Akoglu \\
  Carnegie Mellon University \\
  \texttt{lakoglu@cs.cmu.edu} \\
}

\begin{document}

\maketitle

\begin{abstract}
\input{00abstract.tex}

\end{abstract}

\vspace{-0.15in}
\section{Introduction}
\vspace{-0.05in}
\label{sec:intro}
\input{01intro}

\vspace{-0.15in}
\section{Preliminaries}
\vspace{-0.05in}
\label{sec:prelim}
\input{02prelim}

\vspace{-0.1in}
\section{A Proof-of-Concept for Privileged TFMs}
\vspace{-0.05in}
\label{sec:poc}

\input{025poc}

\vspace{-0.1in}
\section{PIQL: Privileged Information for Quick and Quality Learning for TFMs}
\vspace{-0.05in}
\label{sec:proposed}

\input{03pi}

\vspace{-0.075in}
\subsection{Designing and Training Privileged TFM Architectures}
\vspace{-0.05in}
\label{sec:design}
\input{03arch}

\vspace{-0.1in}
\section{Theoretical Analysis}
\vspace{-0.075in}
\label{sec:theory}

\input{04theoryshort}

\vspace{-0.1in}
\section{Experiments}
\vspace{-0.1in}
\label{sec:experiments}

\input{05experiments}

\vspace{-0.1in}
\section{Conclusion}
\vspace{-0.075in}
\label{sec:conclusion}

\input{07conclusion}

\clearpage
\newpage

\bibliographystyle{plain}
\bibliography{00refs}

\clearpage

\appendix
\section*{Appendix}
\section*{Broader Impact and Limitations}
\vspace{-0.05in}
\input{impactlimitations.tex}

\input{08appendix}

\end{document}

%% file: dfn.tex
\newtheorem{problem}{Problem}

\newcommand{\DX}{\mathcal{D}_X}

\makeatletter
\newcommand\footnoteref[1]{\protected@xdef\@thefnmark{\ref{#1}}\@footnotemark}
\makeatother

\newcommand{\cbit}{\begin{compactitem}}
\newcommand{\ceit}{\end{compactitem}}
\newcommand{\cben}{\begin{compactenum}}
\newcommand{\ceen}{\end{compactenum}}

\newcommand{\beq}{\begin{equation}}
	\newcommand{\eeq}{\end{equation}}

\newcommand{\fomo}{{\sc FoMo-0D}\xspace}

\definecolor{darkgreen}{RGB}{41,166,41}

\newcommand{\bit}{\begin{itemize}}
	\newcommand{\eit}{\end{itemize}}
\newcommand{\ben}{\begin{enumerate}}
	\newcommand{\een}{\end{enumerate}}

\newcounter{x}

\newcommand{\bX}{\mathbf{X}}

\newcommand{\bPs}{\mathbf{P}^\ast}
\newcommand{\bPsi}{\mathbf{P}^{\ast(i)}}

\newcommand{\hatbPsu}{\widehat{\mathbf{P}}_u^\ast}

\newcommand{\bx}{\mathbf{x}}
\newcommand{\by}{\mathbf{y}}
\newcommand{\bs}{\mathbf{s}}
\newcommand{\bz}{\mathbf{z}}
\newcommand{\bp}{\mathbf{p}}

\newcommand{\mcD}{\mathcal{D}}

\definecolor{celadon}{rgb}{0.67, 0.88, 0.69}
\definecolor{carolinablue}{rgb}{0.6, 0.73, 0.89}

\newcommand{\bmu}{\boldsymbol{\mu}}
\newcommand{\bsigma}{\boldsymbol{\sigma}}
\newcommand{\bSigma}{\mathbf{\Sigma}}

\newcommand{\Dtr}{\mathcal{D}_{\rm train}}

\newcommand{\Dts}{\mathcal{D}_{\rm test}}

\newcommand{\Xtr}{\mathbf{X}_{\rm tr}}

\newcommand{\Xts}{\mathbf{X}_{\rm test}}

\newcommand{\Xc}{\mathbf{X}_{\rm c}}

\newcommand{\Xq}{\mathbf{X}_{\rm q}}

\newcommand{\Zsu}{\mathbf{Z}^{\ast}_{\rm u}}
\newcommand{\hatZsu}{\widehat{\mathbf{Z}}^{\ast}_{\rm u}}

\newcommand{\ytr}{\mathbf{y}_{\rm tr}}
\newcommand{\yc}{\mathbf{y}_{\rm c}}
\newcommand{\yq}{\mathbf{y}_{\rm q}}
\newcommand{\hatyts}{\widehat{\mathbf{y}}_{\rm test}}

\definecolor{aliceblue}{rgb}{0.867, 0.917, 0.964}
\definecolor{aliceyellow}{rgb}{0.999, 0.945, 0.796}
\definecolor{alicegray}{rgb}{0.844, 0.867, 0.898}

%% file: 00abstract.tex
Training foundation models is computationally intensive and often slow to converge.
We introduce \method\footnote{\label{fn:url}All source code and model checkpoints are open at \url{https://anonymous.4open.science/r/PIQL}\vspace{-0.1in}}, Privileged Information for Quick and Quality Learning, the first framework to systematically integrate privileged information (PI) 
to simultaneously accelerate learning and improve generalization in tabular foundation models (TFMs). We construct two complementary forms of PI: (i) aggregate dataset-level statistics that reduce the burden on in-context learning, and (ii) encodings of the underlying data-generating program, providing knowledge beyond observable data. 
We further design an architecture that effectively transfers the train-time-only PI by learning to reconstruct it from observed context at inference.
We provide a theoretical analysis characterizing conditions under which PI reduces the population-level approximation gap and accelerates convergence in finite-data regimes.
Empirical evidence shows that \method enables TFMs to achieve faster convergence, lower final loss, and better generalization, in effect, reducing data and compute requirements.
Our work establishes PI-guided pretraining as a principled and practical paradigm for improving the efficiency and performance of foundation models.

%% file: 01intro.tex
Training foundation models (FMs) poses a formidable learning problem. The combination of massive, heterogeneous datasets and high-dimensional parameter spaces substantially slow down convergence, often requiring weeks of intensive computation to achieve robust generalization.

Motivated by the insight that effective guidance can dramatically improve learning efficiency, as captured by the Japanese proverb, ``\textit{Better than a thousand days of self-instruction is one day with a great teacher}'' \citep{vapnik2015blearning}, we investigate a \textbf{novel concept} for training FMs:
incorporating a ``teacher'' that provides high-utility auxiliary information during learning.
Our focus is on tabular foundation models (TFMs), though the underlying principles extend naturally to other modalities.

Our goal for teacher-guided TFMs is twofold:  \textbf{(1)} to accelerate convergence (i.e., learning rate), thereby reducing pretraining cost; and \textbf{(2)} to improve asymptotic generalization, which is central to downstream TFM performance on new tasks.
To this end, we explore a \textbf{new paradigm for learning tabular foundation models} in which additional information is provided to the learner during training.

This raises a key question: 
\textit{Does there exist auxiliary information that can systematically improve both learning speed and generalization?} We posit two complementary mechanisms for constructing such information:  \textbf{(i)} targeting the in-context learning (ICL) capacity of the backbone Transformer, and \textbf{(ii)} leveraging prior knowledge of the data-generating process.

\input{FIG/prelim}

First, TFMs perform ICL by processing training examples as context and predicting labels for test points that cross-attend to the context in a single forward pass. Prior work suggests that
each layer of ICL implements a distinct computational step, contributing incrementally to the final prediction \cite{garg2022transformers, von2023transformers}. %
Motivated by this, we introduce \textit{aggregate dataset-level statistics} as auxiliary inputs, allowing the model to bypass reconstructing global properties from context and thereby increasing the model's effective depth by freeing ICL capacity.
Second, unlike language models pretrained on real-world corpora, TFMs are pretrained on synthetic datasets sampled from known data priors \citep{hollmann2023tabpfn,hollmannv2,shen2025fomod,zhang2025mitra,ding2026outformer}.
That is, the true generating mechanism for each dataset is fully known during pretraining. 
While typical TFM pretraining discards this generative knowledge after sampling, we instead 
supplement the model with this \textit{generative information}, thereby providing knowledge about the underlying process
beyond what can be inferred from observed data alone.

Fig. \ref{fig:prelim} presents striking results from a %
proof-of-concept study (Sec. \ref{sec:poc}), demonstrating the effectiveness of this new paradigm:  progressively richer forms of \textbf{teacher-guided  foundation models achieve faster learning and lower asymptotic loss, reducing both data and computational requirements}.
Further, they produce representations that capture  data distributional properties in earlier layers.

Our approach is grounded in the Learning Using Privileged Information (LUPI) framework \citep{vapnik2009new}, which formalizes learning with auxiliary information available only during training (only). While prior theory characterizes when the so-called privileged information (PI) can accelerate learning \cite{pechyony2010theory}, it does not provide constructive design principles. 
As a result, the design of effective PI remains a largely empirical quest\footnote{Empirical work demonstrated
compelling examples \cite{vapnik2009new}; where pairing handwritten digit images with poetic language descriptions as PI for digit prediction, and supplementing protein sequences with their 3D structure information as PI for protein-family classification improved learning rates.}
We address this gap by proposing concrete, novel forms of PI tailored to TFMs.
The following summarizes our main contributions.

\vspace{-0.05in}
\begin{itemize}[noitemsep, topsep=0pt, leftmargin=1em]
\item {\bf Novel Training Paradigm for TFMs:}
We introduce \method, %
leveraging Privileged Information for Quick and Quality Learning,  the first PI-guided framework for foundation model training toward both accelerated and improved learning.

\item {\bf Constructing Privileged Information (PI) for TFMs:}
Constructing effective PI remains an empirical challenge. We design two novel forms of PI for TFMs: 
 \textbf{(i)} aggregate dataset-level statistics that alleviate the demand on the model's ICL capacity (available during both training and inference), and \textbf{(ii)} encoding of the underlying data-generating program that provides guidance beyond what can be inferred from in-context observations alone (train-time-only). 
\item {\bf Designing Privileged TFM Architecture:}
We design a new  architecture that transfers the train-time-only PI (program encoding) by learning to reconstruct it from observable input during inference, trained with annealed teacher forcing, gradually replacing the original with estimated PI.

\item {\bf Theoretical Analysis:}
We establish the theoretical potential of \method for TFMs by characterizing the conditions under which PI reduces the population-level approximation gap between standard and PI-augmented hypothesis classes. We further analyze finite-data regimes, quantifying the rate at which models realize PI-induced gains in practice.

\item {\bf Empirical Effectiveness:}
 We support our theoretical analysis under idealized assumptions with empirical evidence, showing that our proposed PI mechanisms and architecture for TFMs yield faster convergence, lower final loss, and improved generalization.
\end{itemize}

\vspace{-0.1in}

%% file: FIG/prelim.tex
\definecolor{lightblue}{RGB}{173,216,230}
\definecolor{lightgreen}{RGB}{209,255,189}

\begin{figure}[!t]
\vspace{-0.1in}
\centering
\begin{subfigure}[t]{0.74\textwidth}
    \vspace{0pt}
    \centering  \includegraphics[width=\linewidth]{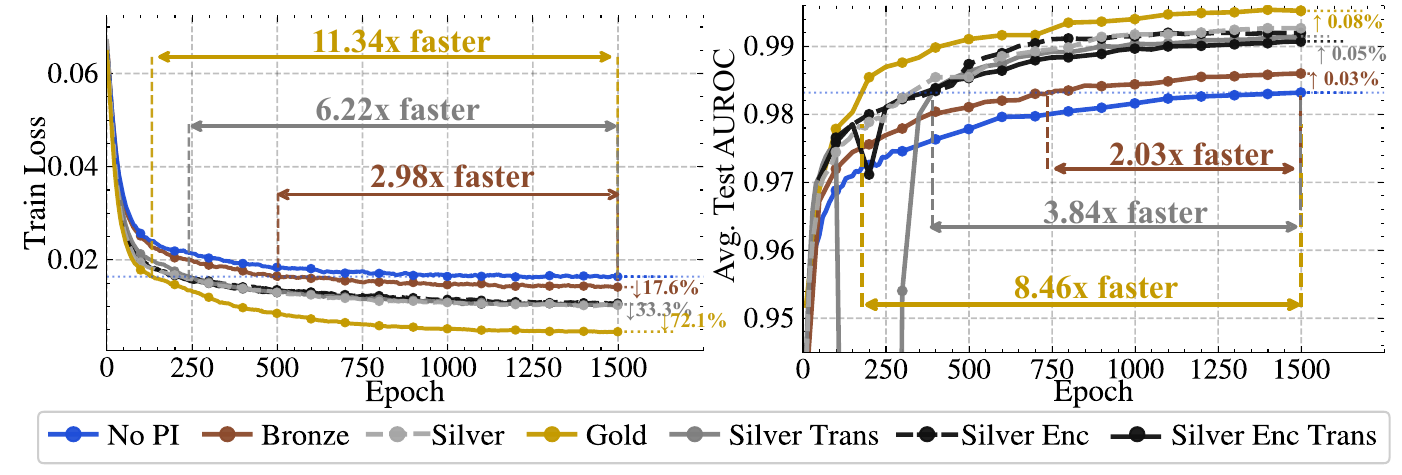}
    \label{fig:training_auroc_prelim}
\end{subfigure}
\hfill
\begin{subfigure}[t]{0.25\textwidth}
    \vspace{0pt}
    \centering
    \scriptsize
    \setlength{\tabcolsep}{2.5pt}
    {\renewcommand{\arraystretch}{0.74}
    \resizebox{0.92\linewidth}{!}{%
    \begin{tabular}{ccc|cp{0.5cm}}
    \toprule
    & \multicolumn{2}{c}{Prior Type} & \multicolumn{2}{c}{GMM \#Clus.} \\
    \cmidrule(lr){2-3} \cmidrule(lr){4-5}
    
L & No-PI & \method & No-PI & \method \\
    \midrule
    9 & \cellcolor{lightgreen}0.789 & 0.996 & 0.609 & 0.655 \\
    8 & 0.788 & 0.985 & 0.606 & 0.625 \\
    7 & 0.768 & 0.981 & \cellcolor{lightgreen}0.547 & 0.556 \\
    6 & 0.750 & 0.976 & 0.502 & 0.536 \\
    5 & 0.748 & 0.975 & 0.472 & 0.537 \\
    4 & \cellcolor{lightblue}0.739 & 0.976 & 0.465 & 0.569 \\
    3 & 0.734 & 0.981 & \cellcolor{lightblue}0.450 & 0.542 \\
    2 & 0.727 & 0.977 & 0.418 & 0.564 \\
    1 & 0.713 & \cellcolor{lightgreen}0.973 & 0.398 & \cellcolor{lightgreen}0.554 \\
    0 & 0.687 & \cellcolor{lightblue}0.740 & 0.363 & \cellcolor{lightblue}0.457 \\
    \bottomrule
    \end{tabular}%
    }}
    \label{tab:layerwise_acc}
\end{subfigure}
\vspace{-0.25in}
\caption{Privileged TFMs pretrained on a GMM-only data prior and equipped with bronze- (empirical dataset-level aggregate statistics), silver- (theoretical model-wise means and variances), and gold-grade (theoretical point-wise mean and variance) PI outperform the baseline (training without PI): \textbf{(left)} they \textit{learn faster and reach lower asymptotic loss}, with progressively larger gains from bronze to gold relative to the no-PI baseline, and \textbf{(middle)} they achieve strong in-distribution generalization earlier, \textit{requiring fewer epochs, thus less data and compute} to reach proficiency. \textbf{(right)} PI \textit{increases effective depth by freeing ICL capacity}: a \method'ed  10-layer TFM trained on a mixture of distinct priors (GMM, SCM-S\&-M, Copula-D\&-P \cite{ding2026outformer}) has linearly probe-able representations for prior type as well as GMM cluster count prediction in \textit{earlier layers} than a no-\method counterpart. (best in color)}
\label{fig:prelim}
\vspace{-0.175in}
\end{figure}

%% file: 02prelim.tex
\subsection{Learning Using Privileged Information (LUPI)}
\label{ssec:lupi} 
\vspace{-0.05in}

The LUPI paradigm \citep{vapnik2009new} departs from the classical learning setting by introducing a ``teacher'' who provides to the learner  \textbf{additional information
during training that is not available at inference time}. Specifically, the learner receives \textit{triplets} $\{(\bx_1, \bx_1^\ast, y_1), \dots, (\bx_n, \bx_n^\ast, y_n)\}$, where $\bx \in \mathcal{X}$ is the primary input \textbf{given by nature}, and $\bx^\ast\in \mathcal{X}^\ast$ is the privileged information (PI) \textbf{constructed by the teacher}. %
In principle, PI serves as ``correcting'' information that accelerates learning by shaping the model's notion of sample difficulty and error structure, and guiding the learner toward better decision boundaries with improved generalization.
The {hypothesis spaces} of the decision (learner) and correcting (teacher) functions, respectively denoted $\mathcal{H}$ and $\Phi$, are \textbf{chosen by the learner}.

While classical supervised learning has been theoretically shown to often converge at a rate of $O(n^{-1/2})$, a teacher that explains structure in the data via PI can accelerate learning to a faster rate of $O(n^{-1})$ under certain conditions \citep{pechyony2010theory};  effectively reducing the complexity of the learning problem.

The utility of the LUPI paradigm
has been demonstrated 
through compelling empirical use cases where PI provides high-level structural context \citep{vapnik2009new,vapnik2017knowledge}. %
In protein function prediction, amino acid sequences (primary input) were supplied with \textbf{3D molecular coordinates} (PI) where  this additional geometric structure %
significantly improved the  performance of sequence-only models. 
In MNIST dataset, hand-written digit images (primary input) were supplemented with short \textbf{human-written poetic descriptions} (PI) capturing stylistic aspects of digits (such as stroke dynamics or visual rhythm) , enabling the learner to more efficiently identify structural invariants among digits. 
In medical diagnosis, 
\textbf{medical-expert generated descriptions and biopsy interpretations} (PI),  complemented
 raw biometric measurements (primary input), leading to improved diagnostic models.

\vspace{-0.075in}
\subsection{Tabular Foundation Models (TFMs)}
\label{ssec:tfms}
\vspace{-0.05in}

TFMs bring the foundation model paradigm to tabular data by pretraining on many \textit{synthetic} tasks sampled from various data-generating priors, including 
Structural Causal Models (SCMs) \citep{hollmann2023tabpfn,hollmannv2}, Gaussian Mixture Models (GMMs), and Copulas \citep{shen2025fomod,ding2026outformer}.
 They learn amortized Bayesian inference,  approximating the posterior predictive distribution across tasks. At inference, a TFM conditions on a labeled training dataset as context and predicts labels for query points in a \textit{single forward pass}, functioning as a pretrained inference mechanism that generalizes through in-context learning \citep{xie2021explanation,pmlr24icl}.

\textbf{Posterior Predictive Distribution (PPD):}
In Bayesian supervised learning, a prior defines a family of data-generating processes $\Psi$. The PPD gives the distribution of labels for a new input $\bx_{\rm test}$ given training data $\Dtr = \{(\bx_1, y_1), \dots, (\bx_n, y_n)\}$. It marginalizes over all hypotheses:
\vspace{-0.065in}
\begin{equation}
    p(y_{\rm test} \mid \bx_{\rm test}, \Dtr) 
    = \int_{\Psi} p(y_{\rm test} \mid \bx_{\rm test}, \psi)\, p(\Dtr \mid \psi)\, p(\psi)\, d\psi.
\label{eq:ppd}
\end{equation}

\vspace{-0.075in}
\textbf{PFNs for PPD approximation:}
Exactly computing the PPD is generally intractable. Prior-data Fitted Networks (PFNs) approximate it by training on synthetic datasets sampled from a prior \citep{muller22}. TFMs exhibit two phases: pretraining and inference.

During \textbf{pretraining}, a hypothesis $\psi \sim p(\psi)$ is sampled and used to generate a dataset $\mcD \sim p(\mcD \mid \psi)$. This dataset is split into a context set $\Dtr$ and targets $\Dts$. A PFN $h_{\boldsymbol{\theta}}$ is trained to predict the labels of test points $(\bx_{\rm test}, y_{\rm test}) \in \Dts$ given $\Dtr$ by minimizing the expected KL divergence between the true predictive distribution and the model \citep{muller22}. 
\begin{equation}
\mathcal{L} 
    = \mathbb{E}_{(\{(\bx_{\rm test}, y_{\rm test})\} \cup \Dtr) \sim p(\mcD)} 
    \big[-\log h_{\boldsymbol{\theta}}(y_{\rm test} \mid \bx_{\rm test}, \Dtr)\big].
 \label{eq:lpfn}
\end{equation}
 In practice, PFNs are implemented with Transformers \citep{VaswaniSPUJGKP17}, where $\Dtr$ serves as context and $\Dts$ as queries, with predictions computed via cross-attention from queries to context.

During \textbf{inference}, the pretrained PFN parameters are fixed. Given a new dataset $\Dtr$, it outputs $h_{\boldsymbol{\theta}}(\cdot \mid \bx_{\rm test}, \Dtr)$ for test inputs in a single forward pass, avoiding dataset-specific training and solely relying on in-context learning. %

\vspace{-0.075in}
\subsection{Notation and Problem Statement}
\vspace{-0.075in}

The theory of LUPI does not offer constructive principles for designing privileged information (PI), therefore the quest for effective PI remains largely empirical (e.g. poetic descriptions of images, expert-annotations of lab results, etc.).
As such, the first problem we address is how to construct useful PI, denoted $\bPs$, for training TFMs. 

\begin{problem}[\textbf{PI Construction}]
Given tabular pretraining datasets $\DX = (\bX,\by)$ drawn from various data priors;  Construct privileged tokens $\bPs$ to prepend to the context window, where $\bPs = (\bPs_a\|\bPs_u)$ is composed of $\bPs_a$ tokens available during both pretraining and inference, and $\bPs_u$ that is available for training but unavailable at inference. 
\end{problem}
\vspace{-0.1in}

The second problem is how to effectively train with the PI, and in particular, how to leverage $\bPs_u$ that is only available during training.

\vspace{-0.05in}
\begin{problem}[\textbf{PI Transfer}]
Given the pretraining corpus $\{\big( \bPsi, \Xc^{(i)}, \yc^{(i)} \;|\; \Xq^{(i)} \big)\}_{i=1}^N$; during training: 1.  Estimate teacher $\phi \in \Phi: \big( \bPs_u, \bPs_a, \Xc, \yc \;|\;  \Xq \big)  \mapsto \yq$
that ingests  privileged context-query pairs and maps to query labels, and 
2. %
Estimate transfer model $s : \big( \bPs_a, \Xc, \yc \big)  \mapsto \Zsu$  that ingests inference-available context-query pairs and maps to  intermediate teacher representations  $\Zsu$ of $\bPs_u$, i.e. inference-unavailable privileged context.
\end{problem}
\vspace{-0.1in}

Note that we jointly optimize the teacher $\phi$ and the transfer model $s$ end-to-end during pretraining, gradually transitioning $\phi$ to $h$ that replaces the true $\bPs_u$ with the student-\textit{predicted} $\hatbPsu$.
At inference, the learner $(h,s) \in \mathcal{H}$ is operates on new train-test pairs using available PI only  by employing (1) $s$ on $\big( \bPs_a, \Xtr,  \ytr \big)$ to obtain $\hatZsu$ and then (2)  $h$  on $\big(\hatZsu, \bPs_a, \Xtr,  \ytr \;|\; \Xts \big)$ to predict $\hatyts$. 

\vspace{-0.075in}
\subsection{Related Work Overview}
\vspace{-0.075in}

Our work extends tabular foundation models (TFMs) with privileged information (PI). Specifically, we design and prepend the training points in the context window with PI tokens. As stated, it appears similar to other learning settings including \textbf{prefix tuning} \citep{li2021prefix}, \textbf{prompt optimization} \citep{lester2021power,chen2023unleashing}, \textbf{instruction tuning} \citep{wei2021finetuned,mishra2022cross}, \textbf{personalization} \citep{adomavicius2005toward,ricci2010introduction}, \textbf{conditional generation} \citep{mirza2014cgan,keskar2019ctrl,ramesh2021zero,rombach2022high}, and \textbf{prediction with auxiliary data} \citep{lim2021tft, zhang2025does,radford2021clip,kim2021vilt,li2023blip2}. 
Appdx. \ref{sec:related} provides a detailed discussion, and Table \ref{tab:relatedcomparison} highlights the key differences between our problem and these existing lines of work. 

In a nutshell, a unifying theme across these paradigms is the use of additional information beyond the raw input to improve learning. However, they differ fundamentally in the nature, availability, and role of this auxiliary information. In particular, \textbf{none of the aforementioned settings is designed for TFMs, nor do they operate under a \textit{train-time-only} information regime}. In contrast, our framework is grounded in the LUPI paradigm, where (part of) PI is available exclusively during training only and is not available at inference time. This distinction is critical: rather than serving as persistent conditioning input (as in personalization, conditional generation, or auxiliary-data learning), or as task-defining signals shared across train and inference (as in instruction tuning and prefix-based methods), our PI acts as a \textit{transient} learning signal that accelerates learning and improves generalization of TFMs. Our auxiliary information, i.e., the carefully-designed PI comprising dataset meta-features and generator program specifications, distinctly encodes  latent structure in the data distribution, enabling more efficient learning of tabular predictors under a foundation model paradigm.

%% file: 025poc.tex
In this section, we demonstrate the premise of leveraging privileged information (PI) in training TFMs using a simple scenario. 
We consider a TFM for outlier detection, namely \fomo \citep{shen2025fomod}, which is trained on datasets synthesized solely from a data prior based on Gaussian Mixture Models (GMMs). 
We construct three different PI of varying strength, so-called \textbf{bronze}-, \textbf{silver}- and \textbf{gold}-grade PI, and show their impact on training loss as well as generalization performance in-distribution, i.e. on hold-out GMM datasets from the data-generating prior.

For each pretraining dataset, inliers are drawn from a GMM with up to $M$$=$$5$ components in up to $D$$=$$100$ dimensions, with centroids $\bmu_{j}^{(k)}$$\in$$[-5,5]$ and diagonal covariance $\bSigma_{jj}^{(k)}$$\in$$(0, 5]$ for $k\in [M]$ and $j \in [D]$.
Outliers are drawn from the ``inflated'' GMM: each Gaussian's variance is scaled by $s'$ within a random subspace $\mathcal{S}^{(k)}$ as $s'\times \Sigma_{jj}^{(k)}$ for $j\in \mathcal{S}^{(k)}$ while keeping the centroids as before.

\vspace{-0.05in}
\begin{itemize}[noitemsep, topsep=0pt, leftmargin=1em]
\item \textbf{Bronze}-grade PI consists of \textbf{(empirical) aggregate statistics at the \textit{dataset-level}}. Specifically, we compute  first- to fourth-order moment statistics (mean, variance, skewness and kurtosis) of each feature across all points in $\Dtr \cup \Dts$ (without access to any labels). The four resulting tokens in $\mathbb{R}^D$ are prepended to the context along with the points in $\Dtr$.

\item \textbf{Silver}-grade PI directly provides the \textbf{(theoretical) mean and variance for \textit{each Gaussian}} that is used to generate inliers and outliers. Let $(\bmu, \bsigma_{\text{ in}}^{2})$ denote the parameters of a Gaussian that generates inliers, and $(\bmu, \bsigma_{\text{ out}}^{2})$ denote those of its inflated version that generates outliers. Then, we prepend the context with two tokens in $\mathbb{R}^D$ per Gaussian of the form $(\bmu, 1/2\bsigma_{\text{ in}}^{2}+1/2\bsigma_{\text{ out}}^{2})$.

\item \textbf{Gold}-grade PI provides the \textbf{(theoretical) mean and variance vector for \textit{each individual point}} from which they were drawn. Specifically, we concatenate each inlier point $\bx_i \sim \mathcal{N}(\bmu^{(k)}, \bsigma^{2,{(k)}}_{\text{ in}})$  as  $(\bx_i \| \bmu^{(k)} \| \bsigma^{2,{(k)}}_{\text{ in}})\in\mathbb{R}^{3D}$. Outliers are also concatenated with the \textit{original} parameters of their respective Gaussian (and \textit{not} the inflated variances).
\end{itemize}
\vspace{-0.075in}

\textbf{Remarks:~} Note that the strength of the provided PI increases with the grade. Bronze-grade PI offers aggregate, empirical feature moments across all points in a dataset. This information can also be computed on a new dataset at inference time. 
Silver-grade PI provides the theoretical generative parameters of the underlying GMM directly. This PI is only available at pretraining time for synthetic datasets, and is hard to infer on a new dataset at inference time; GMM parameter estimation is often done with the Expectation-Maximization algorithm \citep{dempster1977em}, which is iterative (i.e., computationally demanding) and not guaranteed to infer the exact theoretical parameters. Datasets drawn from a GMM exhibit two types of latents: latent parameters of the underlying Gaussians and the latent membership of individual points to the Gaussians. Silver-grade PI offers the first set of latents. Gold-grade PI enhances it by providing the point-level latents, i.e. the (hard) memberships and more specifically, the parameters of the Gaussian that each point was sampled from. 

\vspace{-0.05in}
\textbf{Transferring PI:~}
As both silver- and gold-grade PI provides theoretical GMM latents, they are only available at pretraining time (but hard if not impossible to infer at inference time due to limited depth and data). Therefore, we design a novel Transformer architecture that learns to transfer or reconstruct the PI from available representations at inference time. We defer the details of our proposed privileged TFM architecture and the training algorithm to Sec. \ref{sec:proposed} and we focus on initial but compelling results that showcase the prowess of PI even in this simplified setting of GMM-only data priors.

\vspace{-0.05in}
\textbf{Results:~} Fig. \ref{fig:prelim} presents the (left) training loss curves and (right) average AUROC generalization performance over in-distribution datasets over training epochs,  for the baseline without any PI (blue) and models trained, respectively,  with bronze- (brown), silver- (dashed gray) and gold-grade (yellow) PI. Note that for the train-time-only silver- and gold-grade PI, we provide PI at inference time as a proof of concept. We also include our proposed model with \textit{transferred} silver-grade PI (gray), which reconstructs the PI at inference using only the available inputs. This variant achieves comparable performance, demonstrating the effectiveness of our architecture and its ability to learn to transfer the PI during pretraining.
Also shown are ($i$) the model using the generator program encoding as PI (silver\_enc; dashed black) instead of the prior-level parameters (silver; dashed gray) and ($ii$) the model with the transferred generator PI (black); both of which yield similar gains to those with silver-grade PI.  We defer the details of the program encoding to Sec. \ref{sec:proposed}.

Strikingly, \textbf{privileged training both accelerates loss reduction and yields lower asymptotic loss}, with improvements that grow from bronze to gold relative to the baseline. Moreover, privileged models attain strong in-distribution generalization earlier, requiring fewer training epochs and thus less data and compute. These results highlight the broader promise of privileged foundation models: improved training efficiency and convergence, translating into reduced cost and stronger generalization.

%% file: 03pi.tex
We consider learning under privileged information (PI) for tabular foundation models (TFMs), where the labeled training points in the context window are prepended with additional tokens capturing the PI, shown as follows.
\vspace{-0.45in}
\[
\tikzset{>={Stealth[length=5pt]}}
\begin{tikzpicture}[baseline=(context_all.base),>=Stealth]

 \node (priv_all) at (-4.5, 0) 
{$\displaystyle \underbrace{\bp^\ast_1, \dots, \bp^\ast_k}_{\text{Privileged Context Tokens}}$};

    \node (sep1) at (-2.2, 0) {$;$};

    \node (context_all) at (0,0) 
    {$\displaystyle \underbrace{(\bx_1, y_1), \dots, (\bx_n, y_n)}_{\text{(Labeled) Context Points}}$};

    \node (separator) at (3.1, 0) {$\mid$};

    \node (query_all) at (5.5, 0) 
    {$\displaystyle \underbrace{\bx_{n+1}, \dots, \bx_{n+q}}_{\text{(Unlabeled) Query Points}}$};

    \draw [-Stealth, black, thick] (5.0, 0.5) 
        to [out=165, in=15] 
        node [above=-1pt, font=\small] {Cross-Attention} (0.3, 0.5);

    \path (-0.8, 0.6) -- (-3.7, 0.6) 
    node [midway, above=10pt, font=\small] {Cross-Attn};

    \draw [-Stealth, black, thick] (5.8, 0.8)
        to [out=160, in=0]
        node [above=2pt, font=\small] {} (-3.7, 0.8);

\end{tikzpicture}
\]
\vspace{-0.15in}

We introduce two key innovations: (1) construction of two novel PIs for TFMs (Sec. \ref{sec:pi}), and (2) a new privileged TFM architecture and pretraining algorithm to effectively accommodate train-time-only PI (Sec. \ref{sec:design}).
This section presents how to \textit{operationalize} our innovations, followed by a theoretical analysis in the next section.

\vspace{-0.05in}
\subsection{Constructing Privileged Information for TFMs}
\vspace{-0.05in}
\label{sec:pi}

The LUPI framework does not specify constructive principles for designing PI, leaving the identification of effective PI largely an empirical quest. However, our preliminary results in the previous section suggest that not all PI is equally effective, with clear performance gaps across varying grades.

Toward PI-equipped TFMs, we design two complementary forms of additional information: \textbf{meta-features} and \textbf{generator programs}, distinguished by both \textit{when they are available} and \textit{what role they play}. \textbf{Meta-features} are available at both training and inference time, departing from the strict LUPI setting and instead serving as a practical augmentation. They expose high-level, aggregated statistics of the data directly \textit{in context}, reducing the burden on the TFM to infer such structure via in-context learning.
\textbf{Generator programs}, on the other hand, capitalize on the 
\textit{synthetic-data} regime and the unique setting of tabular foundation models, where the data-generating prior and its parameters are fully known at training time. Available only during training, this form of PI aligns with the original LUPI framework and provides direct access to the underlying true generative process, guiding representation learning and improving generalization beyond what can be achieved from in-context observations, i.e., finite samples from the generative process alone.

Aggregate data statistics (meta-features) can be straightforwardly injected as privileged prefix tokens. In contrast, generator programs require unifying representations across heterogeneous data priors. While Sec. \ref{sec:poc} examined theoretical means and variances as PI for the GMM prior, structural causal models (SCMs) and Copulas exhibit fundamentally different constructs. To enable a unified program prefix, we introduce a program encoder that maps a carefully designed textual–numeric specification of the generator into a sequence of learned representations that can be prepended to the context.

\vspace{-0.075in}
\subsubsection{Meta-Features (weak PI, available at test time)} 
\label{ssec:metapi}
\vspace{-0.05in}

\textbf{Meta-PI} refers to summary statistics or dataset-level aggregate features computed directly from the input dataset and available at both training and inference time. As such, they can be incorporated as standard model inputs. By explicitly providing these aggregates in the context as prefix tokens, meta-PI can free the Transformer from having to reconstruct such information via in-context learning (ICL). This, in turn, frees representational and computational capacity, effectively increasing the usable depth of the TFM and enabling improved task performance.

In this work, we consider meta-PI in the form of moment-based  and quantile-based summary statistics of the features across samples, providing characterizations of the data distribution.
Specifically we use the first four moments (mean, variance, skewness, kurtosis) and five quantiles (5, 25, 50, 75, 90).
These statistics per feature constitute a total of nine prefix tokens $\{\bp_1^\ast, \ldots, \bp_9^\ast\} \in \mathbb{R}^D$. 

Intuitively, the first few moments and quantiles of a distribution can be highly representative and sufficient to distinguish between different distributions \citep{vandervaart1998asymptotic}. Furthermore, summary statistics of a dataset, such as the mean, variance and tail statistics, can aid in outlier detection by providing a baseline for typical values, data dispersion and the distribution of mass across the range, against which atypical observations can be identified.

Note that our goal is not to exhaustively explore the space of possible meta-features to reach highest achievable performance, but rather to demonstrate gains even with simple statistical meta-features; leaving a more systematic exploration of this design space to future work.

\vspace{-0.075in}
\subsubsection{Generator Program (strong PI, train-time only)}
\vspace{-0.05in}
\label{sssec:unifyGenPI}

In contrast to large language models pretrained on massive \textit{real-world} corpora, TFMs occupy a unique setting in which training data are generated \textit{synthetically}. This setting is particularly advantageous, as the data-generating prior and its instantiated parameters define the underlying generative process from which finite samples are drawn. Nonetheless, existing TFM approaches largely ignore this (privileged) information, discarding the generator once training samples are produced.

\textbf{Our work is the first to recognize synthetic data priors as a rich source of information and leverage them within the formal LUPI framework}, where the generator program serves as PI available during training but not at inference, as TFMs are deployed on real-world tasks downstream.

While earlier TFMs used a single family prior such as SCMs \citep{hollmann2023tabpfn}
and GMMs \citep{shen2025fomod}, latest TFMs extend to a mixture of priors  \citep{zhang2025mitra,ding2026outformer}. These mixtures often include \textit{heterogeneous} priors including SCMs, GMMs, Copulas and tree-based priors.
They exhibit fundamentally different constructs, such as means, covariances, etc. for GMMs, the causal DAG and structural equations for SCMs, and individual feature marginals and dependency structure for Copulas. Therefore, the \textbf{central challenge} is to map these heterogeneous specifications into a common token space.

To this end, we introduce a \textbf{program encoder} encompassing two key innovations: First, we propose a serialization format and tokenization that presents a generative specification---the program, its structure, and its parameter values---as a \textit{structured sequence}, analogous to a domain-specific
language (DSL), while \textit{respecting the numeric values} of the parameters. 
Second, we 
encode this sequence using a Transformer, prepending $k$ learnable \texttt{[PREFIX]} tokens to the input. Through self-attention, these tokens accumulate information from the entire generator specification; e.g.:
\begin{equation*}
    \underbrace{[\mathbf{p}_1] \; [\mathbf{p}_2] \; \cdots \; [\mathbf{p}_k]}_{\text{learnable prefix tokens}} \;\; \underbrace{[\texttt{FAMILY:GMM}] \; [\texttt{DIM:5}] \; [\texttt{N\_COMP:3}] \; [\texttt{COMP}] \; [\texttt{W:0.4}] \; \cdots}_{\text{serialized generator specification}}
\end{equation*}
\vspace{-0.15in}

The program encoder is trained on numerous programs from the heterogeneous data priors with varying configurations of the parameters using contrastive loss.
Given a generator configuration 
$\mathcal{G}$, we construct a positive pair by slightly perturbing its parameters as 
$\mathcal{G}' = \operatorname{perturb}(\mathcal{G}, \epsilon)$, 
and encode both as 
$\mathbf{z} = \operatorname{pool}\!\big(\operatorname{Enc}(\operatorname{serialize}(\mathcal{G}))\big)$, and $\mathbf{z}' = \operatorname{pool}\!\big(\operatorname{Enc}(\operatorname{serialize}(\mathcal{G}'))\big)$.
We perform hard negative mining and employ a curriculum to improve contrastive learning. Appdx.~\ref{sec:unigenrep} provides additional details on serialization, tokenization, and encoder training, while Table~\ref{tab:program_encoder_hparams} lists the configurations.

\input{TAB/picompare}
Program encoder training precedes the TFM pretraining. 
Upon its training, each  pretraining generator program $\mathcal{G}$ is serialized and passed through the program encoder, which provides to the TFM the first $k$ output tokens as the \textbf{Generator-PI} prefix tokens; capturing a compact, fixed-size representation of the generator, regardless of the underlying family or  dimensionality. That is,  

\vspace{-0.175in}
$$\{\bp^\ast_1, \dots, \bp^\ast_k\}  = \operatorname{Enc}\!\big(\operatorname{serialize}(\mathcal{G})\big) \;.$$
\vspace{-0.175in}

Table \ref{tab:meta_vs_privileged} summarizes the two PI types for TFMs. Since the generator program is train-time-only PI, we propose a transfer-based TFM architecture trained with annealed teacher forcing, where \textit{estimated} Generator-PI gradually replaces the original, as described next.

%% file: TAB/picompare.tex
\begin{table*}[!t]
\vspace{-0.1in}
\centering
\caption{Meta-feature vs. Generator program based privileged information (PI) for TFMs.}
\vspace{-0.1in}
\hspace{-0.1in}
\scalebox{0.9}{
\begin{tabular}{|p{4.2cm}|p{1.75cm}|p{2.15cm}|p{4.35cm}|}
\hline
\textbf{Type} & \textbf{Availability} & \textbf{Architecture} & \textbf{Role} \\
\hline
Meta-feature (weak PI) & Train \& Test & Prefix tokens & Relieve ICL capacity \\
\hline
Generator program (strong PI) & Train only & Embd. transfer & Expose true generative model  \\
\hline
\end{tabular}
}
\label{tab:meta_vs_privileged}
\vspace{-0.2in}
\end{table*}

%% file: 03arch.tex
\vspace{-0.05in}
\textbf{A Teacher-Student TFM:~}
Let us denote a TFM as 
$h : (\mathcal{D}_{\text{context}}, \bx_{\text{query}}) \mapsto {y}_{\text{query}}$. 
For a privileged TFM, the context during pretraining consists of $\bPs_u$: train-only generator program embedding tokens, $\bPs_a$: inference-available meta-feature tokens, and  $\Dtr$: the training points; i.e. 
$\mathcal{D}_{\text{context}} = (\bPs_u, \bPs_a, \Dtr)$. 
During pretraining, we augment $h$ with a student model denoted $st: (\bPs_a, \Dtr) \mapsto \bPs_u$ that learns to reconstruct the train-only PI from the available context. %

At deployment, the pretrained TFM $(h,st)$ first employs the student model $st$ to obtain an estimation of the inference-unavailable PI, denoted $\hatbPsu$, and then proceeds with the forward pass by the main Transformer backbone $h$ that ingests the student-estimated PI alongside $\bPs_a$ and $\Dtr$.

\vspace{-0.05in}
\textbf{Training to Transfer PI:~}
While consuming only available inputs at inference, we aim for $(h,st)$ to reach a teacher quality model, i.e. $\phi \in \Phi: \big( \bPs_u, \bPs_a, \Xc, \yc \;|\;  \Xq \big)  \mapsto \yq$ that leverages all PI. However, the student model introduces estimation errors that the privileged TFM is to tolerate. 
This phenomenon is commonly known as {exposure bias} and arises from the mismatch between training (where ground-truth is available) and inference (where predictions are fed into the model). 

To mitigate this train-inference discrepancy, we adopt scheduled sampling \citep{bengio2015scheduled}, where train-only PI tokens are progressively replaced by student model predictions during pretraining according to a decaying schedule. Specifically, at each step, model $h$ conditions on either the true tokens $\bPs_u$  or the student predictions $\hatbPsu$, with a probability $\alpha_t$ that is annealed over time/epochs $t$.

Appdx. \ref{sec:pitransfer} provides details of the proposed architecture and training, and a diagram in Fig. \ref{fig:illustration}.

%% file: 04theoryshort.tex
Let $h_\theta$ denote a tabular foundation model and $\mathcal{H} = \{ h_\theta(\mathcal{D}, \bx): \theta \in \Theta \}$ denote the hypothesis class induced by a fixed Transformer
architecture (fixed depth, width, attention heads, and context length). Let $h^*$ be the best decision function (in terms of generalization error) in $\mathcal{H}\textbf{}$, where
$R^* = \inf_{h \in \mathcal{H}} R(h) = R(h^*)$ and $R(h) =
\mathbb{E}_{\psi \sim P_\Psi}
\mathbb{E}_{\mathcal{D} \sim P(\cdot \mid \psi)}
\mathbb{E}_{(\bx,y) \sim \mathcal{D}}
\left[ \ell(h(\mathcal{D}, \bx), y) \right]$.
Let 
$\mathcal{H}_{\text{PI}} =
\{ h_\theta(\mathcal{P}^*, \mathcal{D}, \bx) : \theta \in \Theta_{\text{PI}}\}$
define the PI-augmented hypothesis class, where $\mathcal{P}^*$ denotes the privileged information constructed by the teacher.

PI in foundation models can improve performance
through two fundamentally different mechanisms.

\textbf{I: Representation improvement via meta-features.}
In this regime, PI consists of deterministic dataset-level statistics
$\mathcal{P}^* = \bs(\mathcal{D})$ of the context dataset, which are not reliably computable
by the base Transformer with a fixed  architecture.
As a result, explicitly providing $\mathcal{P}^*$ enlarges the
effective representational capacity of the model at inference time.
This mechanism is  \emph{architectural}: PI reduces approximation
error by bypassing representational bottlenecks in $\mathcal{H}$.

The key assumption is that such meta-features are \emph{useful} for prediction,
i.e., there exists a PI-dependent predictor that achieves lower risk than
the best predictor in the base class. %
To ensure that any observed gain is not due to a mismatch between
hypothesis classes, we also impose an assumption that guarantees that
PI-dependent predictors do not define a strictly richer function class over
$(\mathcal{D}, \bx)$ than $\mathcal{H}$.
The final assumption encodes the architectural bottleneck: the base
transformer cannot approximate $\bs(\mathcal{D})$ arbitrarily well, inducing
a strictly positive representation gap.

\begin{theorem}[Model-Relative Risk Reduction via Meta-PI under Architectural Constraints]
\label{thm:metamain}

Let $\mathcal{P}^* = \bs(\mathcal{D})$ denote meta-features computed from $\mathcal{D}$.
Let the set of useful teachers be defined as
$
\Phi_u = \left\{ \phi \in \Phi \;\middle|\;
R(\phi) \le R(h^*)
\right\}.
$
Assume: (A1) ${\Phi}_u \neq \emptyset$; 
(A2) For each $\phi \in \Phi_u$, there exists  $h\in \mathcal{H}$ such that, for any  $(\mathcal{P}^*, \mathcal{D},  \bx)$ with non-zero probability, 
$\ell_{(\mathcal{P}^*, \mathcal{D})}\!\left(\phi(\mathcal{P}^*, \mathcal{D},  \bx)\right) \geq \ell_{\mathcal{D}}\!\left(h(\mathcal{D},  \bx)\right)$; and 
(A3) There exists $\epsilon > 0$ such that for all
$h \in \mathcal{H}$,
$
\mathbb{E}_{\mathcal{D}  \sim P(\cdot \mid \psi)}
\| \bz_h(\mathcal{D}) - \bs(\mathcal{D}) \|^2
\ge \epsilon,
$
where $\bz_h$ denotes the dataset-level representation by $h$. $\quad$
Then, for some constant $c > 0$, $\;\; R_{\text{PI}}^* \le R^* - c \epsilon \;$.
\end{theorem}
\vspace{-0.05in}

\textbf{II: Information improvement via generator PI.}
In this regime, PI provides auxiliary information about the data generator
$\psi$ during training (only). This information enables the learner to construct a better
estimator of $\psi$, which is encoded into the model parameters.
This mechanism is  \emph{information-theoretic}: PI reduces
uncertainty over latent structure, leading to a better learned predictor.

The key assumption is that this uncertainty is non-zero, i.e.,
$H(\psi \mid \mathcal{D}) > 0$, ensuring that inference over the generator
is non-trivial for the standard learner.

\begin{theorem}[Information-Theoretic Risk Reduction via Generator based PI]
\label{thm:generatormain}
Assume: (A1) and (A2) in Theorem \ref{thm:metamain}, as well as
(A3) The conditional entropy $
H(\psi \mid \mathcal{D})$ is positive, i.e. 
$
H(\psi \mid \mathcal{D})
=
\mathbb{E}_{\mathcal{D}}
\left[
- \int p(\psi \mid \mathcal{D})
\log p(\psi \mid \mathcal{D})
\, d\psi
\right] > 0
$. $\quad$
Then, 
$\;\;R_{\text{PI}}^* < R^*$.
\end{theorem}
\vspace{-0.05in}

Appdx. \ref{sec:theoryappx} presents the detailed statements along with proofs. In short, 
mechanism I improves what the model can \emph{compute at inference time},
while mechanism II improves what the model can \emph{infer during training}.
In both cases, the architecture of the base learner remains fixed; PI
improves performance either by improving access to representations or by
improving inference of latent generative structure.

Through these theorems, Appdx. \ref{ssec:population}
establishes the theoretical potential of our approach, by characterizing the conditions under which PI narrows the \textbf{population-level} approximation gap between a standard learner and a PI-augmented hypothesis class.
In addition, Appdx. \ref{ssec:finite} analyzes the empirical convergence properties of our approach in \textbf{finite-data} regimes, examining the  rate at which the model realizes the benefits of PI in practice.

%% file: 05experiments.tex
\textbf{TFM Backbone, Data, Metrics:} Our proof-of-concept study (Sec. \ref{sec:poc}) considered GMM-only data prior and a  2-layer Transformer backbone. 
Having developed a unified program-as-PI representation across heterogeneous priors, we pretrain a 10-layer outlier detection TFM on a mixture of five (5) different priors \cite{ding2026outformer}. Test set includes 250$\times$5 datasets sampled across these priors, as well as (57) datasets from the real-world benchmark ADBench \cite{han2022adbench}.
We track and report training (cross-entropy) loss, as well as AUROC and AUPRC performances, over epochs.
We perform paired permutation tests to statistically compare two methods and report the $p$-value. See Appdx. \ref{app:setup} for  details.

\vspace{-0.1in}
\subsection{Main Results}
\vspace{-0.05in}

Fig.~\ref{fig:main} shows the training loss (left) and average test AUROC (right) across epochs for the \method-enabled TFM and the no-PI baseline. Consistent with the proof-of-concept results in Fig.~\ref{fig:prelim}, privileged TFMs learn more rapidly, reaching lower training loss and higher AUROC in earlier epochs, and achieve improved asymptotic performance.
Besides the \method'ed model that transfers PI (unified program encoding) during training (red), Fig.~\ref{fig:main} shows the variant trained without transfer (gray). While the training dynamics differ (e.g.,  loss bumps slightly when teacher forcing is annealed),  \method reaches competitive loss and performance at convergence. Appdx. \ref{app:pairedtests} presents the paired tests.

We report per-dataset results on ADBench w.r.t. both AUROC and AUPRC  in Appdx.~\ref{ssec:adb}. Paired tests across checkpoints (see Appdx. Table \ref{tab:adbench_pval}) indicate that the \method-enabled TFM matches the no-PI baseline's performance at earlier epochs, with no statistically significant differences (e.g. $p$$=$$0.299$ for  \textsc{PIQL} @650 vs. base @920 epochs, and $p$$=$$0.119$ for \textsc{PIQL} @700 vs. base @1050 epochs).

\vspace{-0.1in}
\subsection{Ablation Analyses}
\label{sec:ablations}
\vspace{-0.075in}

\input{FIG/mainresults.tex}

\textbf{LLM-as-Program-Encoder:}
We evaluate the utility of our unified program encoder by replacing it with embeddings from off-the-shelf language models. We consider both a general-purpose open-source model, \texttt{Qwen3-Embed-4B} \cite{qwen3_embed_4b}, and a code-specialized, identifier-aware model, \texttt{codet5p-110m-embedding} \cite{wang2021codet5}.
Fig. \ref{fig:ablations} (left) shows that off-the-shelf embeddings do not provide PI benefits, highlighting the importance of our program-specialized encoder.

\vspace{-0.05in}
\textbf{Generator-PI, Meta-PI and Meta-PI++:}
Fig.~\ref{fig:prelim} studied meta-PI (bronze) and generator-PI (silver) in isolation, while the proposed \method-TFM can ingest both. Fig.~\ref{fig:ablations} (middle) shows that augmenting generator-PI with meta-PI yields marginal gains, suggesting that train-time generator-PI provides more effective training signal.
Our initial list contained only 9 meta-features (Sec. \ref{ssec:metapi}).
To strengthen the meta-PI signal, we expand to 100 meta-features via an LLM (see Appdx.~\ref{ssec:llmmeta}). Fig.~\ref{fig:ablations} (right) shows that meta-PI++ using these 100 LLM-generated features alone improves performance, reaching silver-grade levels. These results motivate further research into novel PI designs for FMs.

\input{FIG/ablationfig}

\vspace{-0.1in}
\subsection{Layer-wise Representation Probing}
\vspace{-0.075in}

Finally, we use linear probes \citep{conneau2018you} (diagnostic classifiers) to analyze representations from each layer (0--9) of pretrained privileged vs. standard TFMs. Specifically, we predict two targets: ($i$) type of generator (5-way, balanced) and ($ii$) number of components ($M$$=$$5$-way balanced, for GMM-generated data only). Note that these latent variables are \textbf{not} available directly from the observed context samples at inference. As shown in Fig. \ref{fig:prelim} (right), PIQL'ed TFM representations reflect these latent information \textit{more effectively} (higher classification accuracy) and at \textit{earlier layers} (matching results highlighted). These results provide empirical support for two mechanisms by which PI enhances TFMs: by freeing up ICL capacity and by capturing latent distributional information.

%% file: FIG/mainresults.tex
\begin{figure}[!t]
\vspace{-0.15in}
    \centering
\begin{subfigure}[t]{0.45\textwidth}
    \vspace{0pt}
    \centering
    \includegraphics[width=\linewidth]{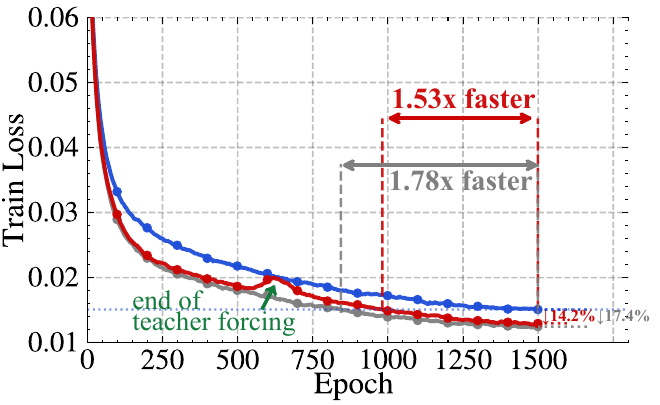}
    \label{fig:training_curve_10layer}
\end{subfigure}
\begin{subfigure}[t]{0.45\textwidth}
    \vspace{0pt}
    \centering
    \includegraphics[width=\linewidth]{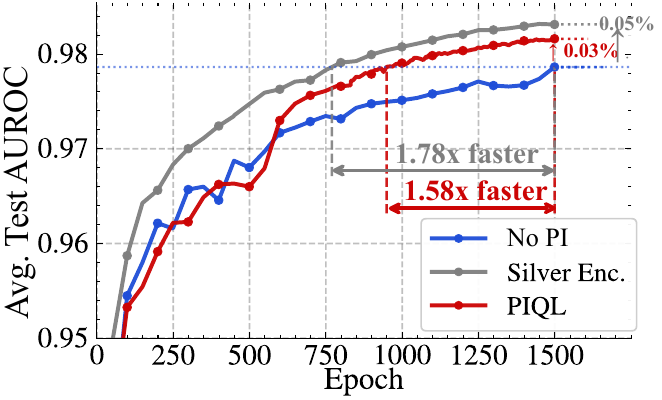}
    \label{fig:avg_auroc_10layer}
\end{subfigure}
    \vspace{-0.2in}
    \caption{Training loss \textbf{(left)} and average test AUROC \textbf{(right)} for the proposed \method model (red) versus the no-PI baseline (blue). \method exhibits accelerated learning, achieving lower loss and higher AUROC in fewer epochs. Despite minor fluctuations during teacher-forcing annealing, \method achieves similar loss and performance to the no-transfer variant (gray) at convergence. (best in color)}
    \label{fig:main}
    \vspace{-0.2in}
\end{figure}

%% file: FIG/ablationfig.tex
\begin{figure}[h]
\vspace{-0.1in}
\centering
\hspace{-0.1in}
\begin{subfigure}[t]{0.33\textwidth}
    \vspace{0pt}
    \centering  \includegraphics[width=\linewidth]{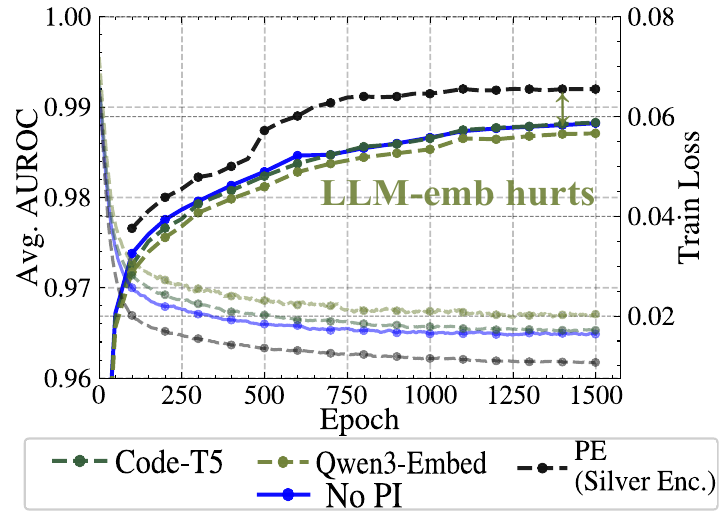}
    \label{fig:llm_ablation}
\end{subfigure}
\begin{subfigure}[t]{0.33\textwidth}
    \vspace{0pt}
    \centering  \includegraphics[width=\linewidth]{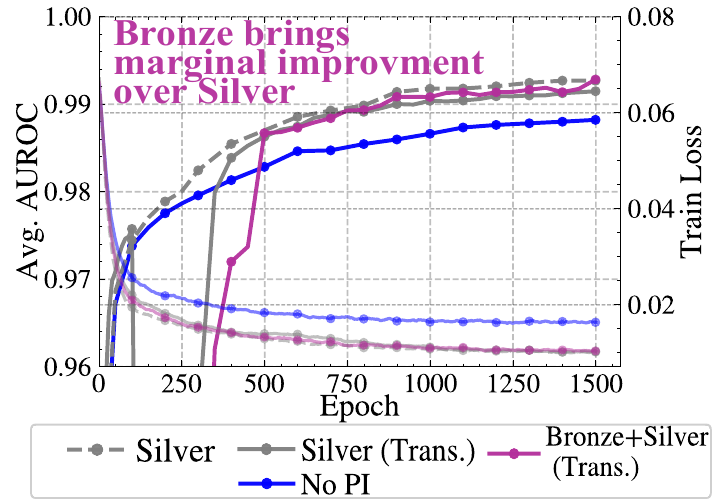}
    \label{fig:bronzesilver}
\end{subfigure}
\begin{subfigure}[t]{0.33\textwidth}
    \vspace{0pt}
    \centering  \includegraphics[width=\linewidth]{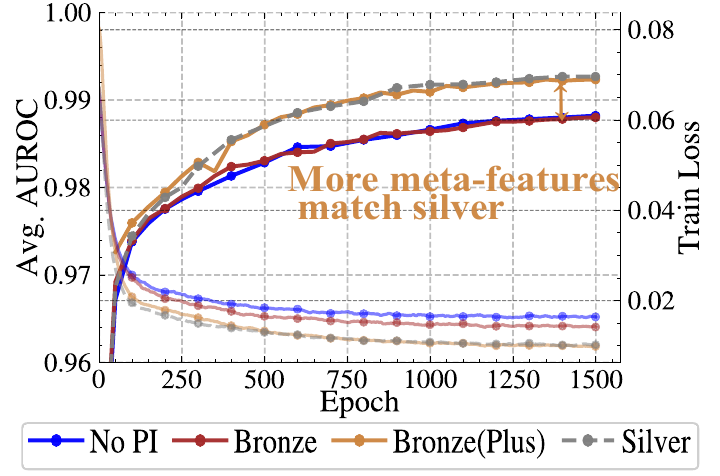}
    \label{fig:bronzeplus}
\end{subfigure}
\vspace{-0.15in}
\caption{\textbf{(left)} LLM-as-Program Encoder; \textbf{(middle)} GeneratorPI + MetaPI; \textbf{(right)} MetaPI++}
\label{fig:ablations}
\vspace{-0.1in}
\end{figure}

%% file: 07conclusion.tex
We introduced \method, Privileged Information for Quick and Quality Learning, a new paradigm   for tabular foundation models (TFMs) in which auxiliary information is provided during training.  We constructed two forms of PI based on meta-features and the data-generating program, supporting two  mechanisms for improving TFMs: freeing up ICL capacity and incorporating latent generative information.  
To operationalize \method, we developed a program encoder that unifies diverse generators and designed an architecture that transfers train-time-only generator-PI via teacher-annealed training. Experiments showed that \method enables TFMs to reach lower loss faster, in effect, reducing data and compute requirements. We expect our work to motivate future research on PI-guided pretraining as a principled and practical approach toward more efficient and effective foundation models.

%% file: impactlimitations.tex
\textbf{Broader Impact:~}
Our work introduces a new paradigm for pretraining foundation models (FMs), with potential to accelerate learning and improve asymptotic loss. The implication of an improved learning rate is fewer training epochs, which translates to reduced compute and data requirements. In the era of increasingly large FMs that follow a rapid scaling trend, reducing pretraining cost and resource demands can help lower barriers to entry for developing such models. It could also contribute to more efficient use of compute infrastructure, with tangible  benefits in energy consumption and environmental impact.

\textbf{Limitations:~} 
Our work focuses specifically on tabular foundation models and leverages their unique property of being trained on synthetic datasets to design generator-specific privileged information (PI). While our study is grounded in this setting, we anticipate that the underlying principles could extend naturally to other types of foundation models as well.
In addition, our work serves as a proof of concept demonstrating the feasibility of this PI-driven paradigm; however, further extensive experiments are needed to assess its generalization across broader settings.
Finally, while the notion of PI is intuitive, its construction is inherently open-ended. In this work, we design two new forms of PI for TFMs, while leaving the exploration of alternative constructions to future work.

%% file: 08appendix.tex
\section{Extended Related Work}
\label{sec:related}

\input{06related}

\section{Unifying Generator Programs in a Shared Token Space}
\label{sec:unigenrep}

\subsection{Goal and Challenges}
We aim to \textbf{encode the generator specification}, consisting of the program, its structure, and its parameters, rather than the data it emits. Concretely, a generator configuration is derived from the following synthetic data priors:
\begin{itemize}
    \item A {Gaussian Mixture Model (GMM)} with specific parameters (number of components, means, covariances, mixture weights),
     \item A {Structural Causal Model (SCM)} with a specific DAG and structural equations, and
    \item A {Copula} with specific feature marginals and dependence structure.
\end{itemize}
Our \textbf{goal} is to represent any such generator configuration as a small set of vector tokens. That is, we seek to encode the \emph{generative program itself} into a compact token representation.

This poses a nontrivial challenge, as the three generator families are fundamentally different objects:
\begin{itemize}
    \item \textbf{GMM:} defined by $M$ Gaussian clusters, each with a mean vector $\boldsymbol{\mu}_k \in \mathbb{R}^d$, a covariance matrix $\boldsymbol{\Sigma}_k \in \mathbb{R}^{d \times d}$, and a mixture weight $\pi_k$;
     \item \textbf{SCM:} defined by a DAG adjacency matrix $\mathbf{A} \in \{0,1\}^{d \times d}$, structural equations with specified functional forms, and per-variable noise distributions;
    \item \textbf{Copula:} defined by a copula family with associated parameters, together with $d$ marginal distributions, each specified by its own family and parameters.  
\end{itemize}
In addition, each generator family contains distinctive outlier types. Overall, our  mixture consists of five priors, which we briefly describe in the following. We refer to \cite{ding2026outformer} for detailed descriptions. 

\begin{itemize}[noitemsep, topsep=0pt]
    \item[($i$)]  For GMM, outliers are generated by inflating the covariance matrices of subspaces and sampling from the resulting inflated Gaussian distributions.
\end{itemize}

For SCM, there are two types of outlier archetypes: ($ii$) measurement and ($iii$) structural outliers.

\begin{itemize}[noitemsep, topsep=0pt]
    \item[($ii$)]  The former picks one node from the DAG and samples with inflated exogenous noise; 
    \item[($iii$)] The latter breaks the causal dependencies within the DAG by removing or reversing edges. 
\end{itemize}

Copulas also exhibit two types of outlier archetypes: ($iv$) dependence and ($v$) probabilistic outliers.

\begin{itemize}[noitemsep, topsep=0pt]
    \item[($iv$)]  The former samples outliers by overwriting the copula vector and inverting a subset of dimensions; 
    \item[($v$)] The latter samples outliers by selecting a subset of features and perturbing the copula coordinates by pushing towards the boundaries. 
\end{itemize}

The central \textbf{challenge} is to map these five heterogeneous specifications into a common token space, i.e., unify the distinct prior programs.

\subsection{Program Serialization and Tokenization}

\textbf{Serialization:} We treat each generator specification as a \textbf{structured sequence}, analogous to a domain-specific language (DSL) or program. Below we illustrate the structured format for each generator family.

\begin{tcolorbox}[
colback=blue!5!white,
colframe=blue!75!black, 
title=\textbf{Example of GMM Prior Description.},fonttitle=\bfseries,
  fontupper=\small 
]
\begin{verbatim}
[FAMILY:GMM] [DIM:5] [N_COMP:3]
[ENTITY] [TYPE:COMPONENT] [ID:0] [MEAN:-1.9,0.0,1.6,...] [COV_DIAG:0.8,0.5,...]
[ENTITY] [TYPE:COMPONENT] [ID:1] [MEAN:0.0,-0.2,-0.6,...] [COV_DIAG:0.9,0.4,...]
[ENTITY] [TYPE:COMPONENT] [ID:2] [MEAN:0.1,-0.9,0.0,...] [COV_DIAG:0.1,0.7,...]
[OUTLIER] [TYPE:inflated_cov] [SUB_DIMS:1.0,2.0...]
[ENTITY] [TYPE:OUTLIER_COMPONENT] [ID:0] [INF_COV:0.8,2.5,...]
[ENTITY] [TYPE:OUTLIER_COMPONENT] [ID:1] [INF_COV:0.9,2.3,...]
[ENTITY] [TYPE:OUTLIER_COMPONENT] [ID:2] [INF_COV:0.1,3.5,...]
\end{verbatim}
\end{tcolorbox}

The program above encodes a 5-dimensional GMM prior with three inlier components, where each component is described by its mean vector and diagonal covariance. The outlier mechanism is specified as inflated covariance on selected sub-dimensions, meaning outliers are generated by increasing variance along certain feature dimensions for each corresponding mixture component.

\begin{tcolorbox}[
colback=blue!5!white,
colframe=blue!75!black, 
title=\textbf{Example of SCM-Measurement Prior Description.},fonttitle=\bfseries,
  fontupper=\small 
]
\begin{verbatim}
[FAMILY:SCM] [DIM:4]
[ENTITY] [TYPE:VAR][ID:0] [LAYER:0] [EQ:tanh] [PARENTS:none] [COEFFS:none]
[ENTITY] [TYPE:VAR][ID:1] [LAYER:0] [EQ:tanh] [PARENTS:none] [COEFFS:none]
[ENTITY] [TYPE:VAR][ID:2] [LAYER:1] [EQ:leaky_relu] [PARENTS:1.0] [COEFFS:-0.7]
[ENTITY] [TYPE:VAR][ID:3] [LAYER:2] [EQ:tanh] [PARENTS:0.0,2.0] [COEFFS:1.4,-0.1]
[OUTLIER] [TYPE:noise_scale] [HIGH_NOISE:5.0] [HIGH_NOISE_PROB:0.2][N_PROTO:2]
[ENTITY] [TYPE:OUTLIER_PROTOTYPE] [ID:0] [HIGH_NOISE_NODES:1.0,2.0]
[ENTITY] [TYPE:OUTLIER_PROTOTYPE] [ID:1] [HIGH_NOISE_NODES:2.000]
\end{verbatim}
\end{tcolorbox}

The program above  encodes a 4-dimensional SCM, where variables are generated according to a layered causal graph: variables 0 and 1 are root nodes, variable 2 depends on variable 1 through a leaky-ReLU equation, and variable 3 depends on variables 0 and 2 through a tanh equation. The outlier mechanism is noise-scale perturbation, where selected variables receive high exogenous noise with probability 0.2, producing two outlier prototypes: one affecting nodes 1 and 2, and another affecting node 2 only.

\begin{tcolorbox}[
colback=blue!5!white,
colframe=blue!75!black, 
title=\textbf{Example of SCM-Structural Prior Description.},fonttitle=\bfseries,
  fontupper=\small 
]
\begin{verbatim}
[FAMILY:SCM] [DIM:4]
[ENTITY] [TYPE:VAR][ID:0] [LAYER:0] [EQ:tanh] [PARENTS:none] [COEFFS:none]
[ENTITY] [TYPE:VAR][ID:1] [LAYER:0] [EQ:tanh] [PARENTS:none] [COEFFS:none]
[ENTITY] [TYPE:VAR][ID:2] [LAYER:1] [EQ:leaky_relu] [PARENTS:1.0] [COEFFS:-0.7]
[ENTITY] [TYPE:VAR][ID:3] [LAYER:2] [EQ:tanh] [PARENTS:0.0,2.0] [COEFFS:1.4,-0.1]
[OUTLIER] [TYPE:weight_mask] [PERTURB_PROB:0.2]  [N_PROTO:2]
[ENTITY] [TYPE:OUTLIER_PROTOTYPE] [ID:0] [FLIPPED_NODES:0.0] [DROPPED_NODES:2.0]
[ENTITY] [TYPE:OUTLIER_PROTOTYPE] [ID:1] [FLIPPED_NODES:1.0] [DROPPED_NODES:]
\end{verbatim}
\end{tcolorbox}

The above encodes a similar 4-dimensional SCM: variables 0 and 1 are root variables, variable 2 depends on variable 1, and variable 3 depends on variables 0 and 2. The outlier mechanism is a weight-mask perturbation, where causal relationships are altered with probability 0.2 by flipping or dropping selected nodes/edges. Specifically, prototype 0 flips node 0 and drops node 2, while prototype 1 flips node 1 without dropping any node, creating structural outliers that violate the normal causal generation process.

\begin{tcolorbox}[
colback=blue!5!white,
colframe=blue!75!black, 
title=\textbf{Example of Copula-Probabilistic Prior Description.},fonttitle=\bfseries,
  fontupper=\small 
]
\begin{verbatim}
[FAMILY:COPULA] [DIM:4]
[ENTITY] [TYPE:COPULA_BASE] [ID:0] [COPULA_PARAM:random] [CHOL:1.000,0.852,...]
[ENTITY] [TYPE:MARGINAL] [ID:0] [DIST:studentt] [PARAMS:8,0,1]
[ENTITY] [TYPE:MARGINAL] [ID:1] [DIST:beta] [PARAMS:1.2,3.0]
[ENTITY] [TYPE:MARGINAL] [ID:2] [DIST:normal] [PARAMS:0.7,1.6]
[ENTITY] [TYPE:MARGINAL] [ID:3] [DIST:exponential] [PARAMS:1.8]
[OUTLIER] [TYPE:PROBABILISTIC] [N_PROTO:2] [STRENGTH_RANGE:0.2~0.4]
[ENTITY] [TYPE:OUTLIER_PROTOTYPE] [ID:0] [PERTURBED_DIMS:0.0,1.0,2.0]
[ENTITY] [TYPE:OUTLIER_PROTOTYPE] [ID:1] [PERTURBED_DIMS:1.0]
\end{verbatim}
\end{tcolorbox}

The above encodes a 4-dimensional copula prior, where the dependence structure is specified by a copula base through a Cholesky factor, while each dimension has its own distinct marginal distribution: Student-t, Beta, Normal, and Exponential. The outlier mechanism is probabilistic perturbation, where selected dimensions are perturbed with a strength sampled from 0.2–0.4. Specifically, prototype 0 perturbs dimensions 0, 1, and 2, while prototype 1 perturbs only dimension 1, creating outliers that deviate from the normal marginal/dependence pattern.

\begin{tcolorbox}[
colback=blue!5!white,
colframe=blue!75!black, 
title=\textbf{Example of Copula-Dependence Prior Description.},fonttitle=\bfseries,
  fontupper=\small 
]
\begin{verbatim}
[FAMILY:COPULA] [DIM:3]
[ENTITY] [TYPE:COPULA_BASE] [ID:0] [COPULA_PARAM:random] [CHOL:0.83,0.25,0.216]
[ENTITY] [TYPE:MARGINAL] [ID:0] [DIST:normal] [PARAMS:-0.88,1.76]
[ENTITY] [TYPE:MARGINAL] [ID:1] [DIST:interp] [PARAMS:lo=-5.3,hi=6.41,u_grid=2000]
[ENTITY] [TYPE:MARGINAL] [ID:2] [DIST:exponential] [PARAMS:1.4]
[OUTLIER] [TYPE:DEPENDENCE] [N_PROTO:1] [STRENGTH_RANGE:0.97~0.99] 
[ENTITY] [TYPE:OUTLIER_PROTOTYPE] [ID:0] [DISTURBED_DIMS:0.0,2.0]
\end{verbatim}
\end{tcolorbox}

The description above encodes a similar 3-dimensional copula prior, where the copula base defines the dependence structure through a Cholesky factor. Each dimension has its own marginal distribution: Normal, Interpolated/Custom, and Exponential. The outlier mechanism is a dependence perturbation, meaning that outliers are generated by strongly disrupting the dependency pattern among selected dimensions. In this case, the single outlier prototype disturbs dimensions 0 and 2 with high strength, creating samples that may look marginally plausible but violate the normal cross-feature dependence.

\textbf{Tokenization:} With this \textbf{structured/serialized format}, we next describe how to \textbf{tokenize} the program. 

Taking the first GMM as an example, the GMM prior is first serialized as a bracketed program string, where each block, such as \texttt{[FAMILY:GMM]}, \texttt{[ENTITY]}, \texttt{[MEAN:...]}, explicitly marks the semantic role of the information. The tokenizer then parses each bracket into typed tokens: categorical items like \texttt{FAMILY}, \texttt{GMM}, \texttt{ENTITY}, and \texttt{COMPONENT} become symbol tokens; scalar fields like \texttt{DIM}, \texttt{N\_COMP},and \texttt{ID} become scalar tokens; vector fields like \texttt{MEAN} and \texttt{COV\_DIAG} become one vector entry token per dimension; and \texttt{SUB DIMS} becomes index entry tokens indicating which feature dimensions are affected. The box below shows an example tokenization of the GMM example above.

\begin{tcolorbox}[
colback=blue!5!white,
colframe=blue!75!black,
title=\textbf{Example of Tokenized GMM Prior Description.},
fonttitle=\bfseries,
fontupper=\small
]
\begin{verbatim}
[
  {"type": "symbol", "name": "FAMILY"},
  {"type": "symbol", "name": "GMM"},
  {"type": "symbol", "name": "DIM"},
  { "type": "scalar","field": "DIM", "value": 5.0, "family": "GMM", 
  "block": null
  },
  {"type": "symbol", "name": "N_COMP"},
  {"type": "scalar","field": "N_COMP", "value": 3.0, "family": "GMM", 
  "block": null
  },
  {"type": "symbol", "name": "ENTITY"},
  {"type": "symbol", "name": "TYPE"},
  {"type": "symbol", "name": "COMPONENT"},
  {"type": "symbol", "name": "ID"},
  {"type": "scalar","field": "ID", "value": 0.0, 
  "family": "GMM", "block": "ENTITY", "entity_type": "COMPONENT","entity_id": 0
  },
  {"type": "symbol", "name": "MEAN"},
  {"type": "vector_entry","field": "MEAN","dim": 0,"value": -1.9, "family": "GMM", 
  "block": "ENTITY", "entity_type": "COMPONENT","entity_id": 0
  },
  "... remaining COMPONENT and OUTLIER_COMPONENT tokens omitted ..."
]
\end{verbatim}
\end{tcolorbox}

Given the \textbf{tokenized} prior description, each token is converted into a \textbf{dense embedding} by summing learnable embeddings for its semantic attributes. For example, consider the token:
\begin{verbatim}
{
  "type": "vector_entry",
  "field": "MEAN",
  "dim": 0,
  "value": -1.9,
  "family": "GMM",
  "block": "ENTITY",
  "entity_type": "COMPONENT",
  "entity_id": 0
}
\end{verbatim}
This token represents ``dimension 0 of the mean vector for component 0 in a GMM prior''. Its embedding is obtained by summing several pieces: 
(i) field embedding for \texttt{MEAN}, 
(ii) family embedding for \texttt{GMM}, 
(iii) entity-type embedding for \texttt{COMPONENT}, 
  (iv) entity-ID embedding for component \texttt{0}, 
  (v) dimension embedding for dim \texttt{0}, and 
 (vi) value embedding for \texttt{-1.9}. Importantly, the numeric value \texttt{-1.9} is first normalized and then passed through a small MLP to produce its embedding value. Therefore, the final token embedding does not only encode the number itself;  it also encodes what that number means within the prior structure.

 To build a meaningful embedding table, we also utilize a bootstrap step in the beginning to build vocabulary before training, by sampling many synthetic prior descriptions across GMM, SCM, and Copula families, tokenizing them, and collecting all possible symbols, fields, family names, and entity types needed to initialize the embedding tables. Without this step, some tokens appearing later could be missing from the vocabulary and would collapse to the default unknown/PAD index, losing their semantic meaning. The bootstrap is saved and then utilized at TFM pre-training/ inference phase. 
 
 We remark that this special serialization and tokenization of programs produces more meaningful signals than simply using off-the-shelf  text embeddings from language models. (See our LLM-as-Program-Encoder ablations in Sec. \ref{sec:ablations}.) It also embeds more specific program information than existing scientific tokenization methods such as xVal \citep{golkar2023xval}.

\textbf{Encoding:} Finally, we encode the serialized and tokenized sequence using a Transformer, prepending $k$ learnable \texttt{[PREFIX]} tokens (with different embedding initialization) to the input. Through self-attention, these summary tokens accumulate information from the entire generator specification:
\begin{equation*}\underbrace{[\mathbf{p}_1] \; [\mathbf{p}_2] \; \cdots \; [\mathbf{p}_k]}_{\text{learnable prefix tokens}} \;\; \underbrace{[\texttt{FAMILY:GMM}] \; [\texttt{DIM:5}] \; [\texttt{N\_COMP:3}] \; [\texttt{COMP}] \; [\texttt{W:0.4}] \; \cdots}_{\text{serialized generator specification}}
\end{equation*}

A full generator sequence $\mathcal{G}$ is passed through the Transformer encoder, and only the first $k$ output positions are retained:
\begin{equation*}
    \{\bp^\ast_1, \dots, \bp^\ast_k\}  = \operatorname{Enc}\!\big(\operatorname{serialize}(\mathcal{G})\big) \;\in\; \mathbb{R}^{d_{\text{model}}}
\end{equation*}

These $k$ vectors constitute the \emph{generator tokens}; a compact, fixed-size representation of the generator, regardless of the underlying family or dimensionality.

\subsection{Program Encoder Training}
\paragraph{Program Contrastive Loss.}
We train the program encoder using a margin-based contrastive objective over serialized prior programs. For an anchor generator $\mathcal{G}$, we construct a positive $\mathcal{G}^{+}$ by weakly perturbing its parameters while preserving the same prior family, outlier mechanism, and dimensionality. Negatives are either easy negatives sampled from different prior families, or hard negatives generated by applying a stronger perturbation to the same anchor configuration. After serialization, tokenization, and encoding, we obtain embeddings $\mathbf{z}$, $\mathbf{z}^{+}$, and $\{\mathbf{z}^{-}_i\}_{i=1}^{M}$.

We optimize a weighted triplet-style objective:
\begin{equation*}
    \mathcal{L}_{\text{triplet}}
    =
    \frac{1}{M}
    \sum_{i=1}^{M}
    w_i
    \left[
        \|\mathbf{z}-\mathbf{z}^{+}\|_2
        -
        \|\mathbf{z}-\mathbf{z}^{-}_i\|_2
        +
        m_i
    \right]_{+},
\end{equation*}
where hard negatives use a separate margin $m_{\text{hard}}$ and weight $w_{\text{hard}}$, while easy negatives use margin $m_{\text{easy}}$ and unit weight. We further add a hardest-negative mining term
\begin{equation*}
    \mathcal{L}_{\text{mine}}
    =
    \left[
        \|\mathbf{z}-\mathbf{z}^{+}\|_2
        -
        \min_i \|\mathbf{z}-\mathbf{z}^{-}_i\|_2
        +
        m_{\text{hard}}
    \right]_{+}.
\end{equation*}
Finally, to preserve high-level prior-family structure, we introduce learnable family proxies and apply a cross-entropy loss over normalized anchor-proxy similarities. The final objective is
\begin{equation*}
    \mathcal{L}
    =
    \mathcal{L}_{\text{triplet}}
    +
    \lambda_{\text{mine}}\mathcal{L}_{\text{mine}}
    +
    \lambda_{\text{proxy}}\mathcal{L}_{\text{proxy}}.
\end{equation*}

\paragraph{Program Encoder Architecture and Training.}
We train two versions of the program encoder: the first version is for GMM prior only and the second version is for Mix-priors. The GMM-only program encoder contains one \texttt{PREFIX} token while the latter contains $k=10$ learnable prefix tokens to capture more information. We list the detailed hyperparameter configurations of the Transformer based encoders in Table \ref{tab:program_encoder_hparams}. 

Both program encoders are trained for $20$ epochs with $80$ steps per epoch, batch size $16$, AdamW optimizer with learning rate $10^{-4}$, warmup ratio $0.1$, and minimum learning-rate ratio $5\times 10^{-6}$. The contrastive training setup uses $8$ negatives per anchor. The prior-generation range uses dimensionality $d\in[2,100]$, GMM components in $[1,5]$, maximum mean magnitude $6$, maximum variance $6$, covariance inflation scale $5.0$, and $16$ outlier prototypes. The training takes places over one single L40 GPU for approximately $24$ hours.

\paragraph{Curriculum Hyperparameters.} The GMM-only program encoder is not trained with easy/hard negatives due to all negatives being within the same prior family. However, we utilize a curriculum for the Mixed-prior program encoder. For curriculum-based contrastive training, we use a warm-up period of $3$ epochs before gradually increasing the hard-negative difficulty. The curriculum is controlled by the hard-negative sampling ratio $\rho_{\text{hard}}$, which is annealed from $\rho_{\text{start}}=0.1$ to $\rho_{\text{end}}=0.5$, and the hard-negative perturbation strength $\epsilon_{\text{hard}}$, which is annealed from an easier initial value to the target value $\epsilon_{\text{hard}}=0.25$. We set easy margin $m_{\text{easy}}=0.8$, hard margin $m_{\text{hard}}=0.4$, hard-negative weight $w_{\text{hard}}=1.8$, hard-mining weight $\lambda_{\text{mine}}=0.5$, proxy loss weight $\lambda_{\text{proxy}}=0.5$. This schedule exposes the encoder to easier negatives early in training and progressively introduces harder, more semantically similar negatives, encouraging stable coarse-to-fine representation learning.

\begin{table}[t]
\centering
\caption{Hyperparameters of the two program encoder (PE) architecture variants.}
\label{tab:program_encoder_hparams}
\scalebox{0.95}{
\begin{tabular}{lcc}
\toprule
\textbf{Hyperparameter} & \textbf{GMM-only PE} & \textbf{Mixed-prior PE} \\
\midrule
Number of PREFIX tokens & 1 & 10 \\
Maximum sequence length & 2048 & 2048  \\
Hidden dimension $d_{\text{model}}$ & 256 & 256 \\
Number of attention heads & 8 & 8 \\
Number of Transformer layers & 6 & 4 \\
Feedforward dimension & 256 & 256 \\
Dropout rate & 0.1 & 0.1 \\
\bottomrule
\end{tabular}
}
\end{table}

\begin{figure}[htbp]
    \centering
    \includegraphics[width=0.9\linewidth]{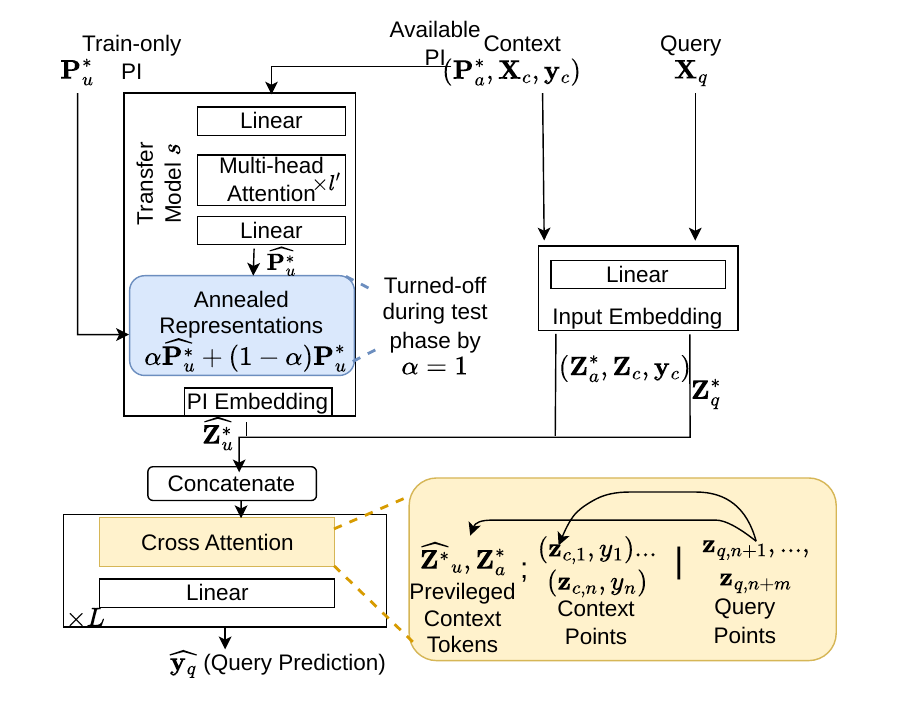}
    \caption{Illustration of our PIQL diagram: the teacher-student TFM and training to transfer PI.}
    \label{fig:illustration}
\end{figure}

\section{Privileged Transformer Architecture and Training with PI Transfer}
\label{sec:pitransfer}

Inspired by the knowledge transfer framework \cite{vapnik2017knowledge}, we propose \method to transfer generative privileged information into the foundation model. The model consists of two coupled components: 
\textbf{(1)} a transferring student model $s$ that learns the embedding of train-time-only
generative privileged information based on $(\mathbf{P}^*_a, \mathbf{X}_c, \mathbf{y}_c)$, and \textbf{(2)} a Transformer-based foundation
model (FM) $h$ that performs context-to-query prediction over an assembled token
sequence.

The transfer model $s$ takes the train-only privileged signal $\mathbf{P}^{*}_{u}$ and $(\mathbf{P}^*_a, \mathbf{X}_c, \mathbf{y}_c)$ to
map it to a set of hidden PI embeddings $\widehat{\mathbf{Z}^*_u}$, which are utilized as prepended context tokens into the transformer model. Since $\mathbf{P}^{*}_{u}$ is un-available during inference time, we aim to reduce the model's dependence
on this privileged information when learning the embedding $\widehat{\mathbf{Z}^*_u}$. To that end, we introduce a representation annealing mechanism.
Specifically, the privileged representation at training step $t$ is formed as
\begin{equation}
    \label{eq:mix}\widehat{\mathbf{P}}_{u}^{(t)}
    =
    (1-\alpha_t)\mathbf{P}^{*}_{u}
    +
    \alpha_t \widehat{\mathbf{P}^{*}_{u}},
\end{equation}
where $\widehat{\mathbf{P}^{*}_{u}}$ is the model-predicted approximation of the
privileged representation and $\alpha_t$ is increased during training. As such,
early training relies primarily on the true  privileged signal $\mathbf{P}^{*}_{u}$,
while later training increasingly relies on the learned approximation
$\widehat{\mathbf{P}^{*}_{u}}$.

The mixed representation $\widehat{\mathbf{P}}_{u}^{(t)}$ in Eq. \eqref{eq:mix} is embedded into
privileged tokens $\widehat{\mathbf{Z}_u^*}$. These tokens are concatenated
with available privileged/context tokens $\mathbf{Z}^{*}_{a}$, labeled context
tokens $\{(\mathbf{z}_{c,i}, y_i)\}_{i=1}^{n}$, and query tokens
$\{\mathbf{z}_{q,j}\}_{j=1}^{m}$:
\begin{equation}
    \mathbf{S}
    =
    \operatorname{Concat}
    \left(
    \widehat{\mathbf{Z}_u^*}, \; 
    \mathbf{Z}^{*}_{a}, \; 
    \{(\mathbf{z}_{c,i}, y_i)\}_{i=1}^{n}, \; 
    \{\mathbf{z}_{q,j}\}_{j=1}^{m}
    \right).
\end{equation}
The assembled sequence is then processed by the Transformer backbone  $h$(similar to the main backbone in \cite{ding2026outformer,shen2025fomod}). The context
tokens provide labeled in-context examples, while the privileged tokens provide
additional information about the data-generating process during training. The
decoder is applied only to the query-token outputs to obtain the final query
predictions:
\begin{equation}
    \widehat{\mathbf{y}}_{q}
    =
    h_{\theta}(\mathbf{S})_{\mathrm{query}} \;.
\end{equation}

\paragraph{Hyperparameters.}
We use a lightweight transfer model $s$ with $l^\prime=1$ to limit the number of additional parameters. It projects the 100-dimensional input features into a 256-dimensional space, where learned query tokens perform attention-based pooling over the input sequence. The pooled representation is then projected to a compact 256-dimensional output. The other detailed hyperparameter configurations can be found in the additional experiment details in Appendix \ref{sec:extraexp}.

\section{Theoretical Analysis}
\label{sec:theoryappx}

\input{04theory}

\section{Additional Experiments and  Details}
\label{sec:extraexp}

\input{05extraexp}

%% file: 06related.tex
\input{TAB/related}

Our work extends tabular foundation models (TFMs) with privileged information (PI). Specifically, we design and prepend the training points in the context window with PI tokens. As stated, it appears similar to other learning settings including prefix tuning, instruction tuning, personalization, predictions with auxiliary data, and conditional generation. Table \ref{tab:relatedcomparison} highlights the key differences between our problem and these existing lines of work, on which we elaborate in this section.

\textbf{Prefix-tuning} \citep{li2021prefix} aims to learn different soft-prompts per task while keeping the parameters of a pretrained model intact. As opposed to maintaining multiple model parameters, one per task, prefix tokens require a much smaller number of parameters.

\textbf{Soft-prompt optimization} \citep{lester2021power,chen2023unleashing} is created concurrently and shares the same objective. The difference is that prefix-tuning prepends a sequence of learned prefixes at every transformer layer while prompt tuning is simpler and limits prefixes to the input layer.
In both settings, \textit{the prepended tokens are learned} to optimize a specific task's performance, and unlike our setting,  there exist no external auxiliary input.

\textbf{Instruction tuning} \citep{wei2021finetuned,mishra2022cross}, in contrast, fine-tunes a pretrained model's parameters on a collection of tasks expressed via natural-language instructions.
Each training example includes an instruction that specifies the task, in addition to the input/output pairs, and the model is trained typically with a unified language modeling objective to follow these instructions and generalize to new ones. 
Unlike our train-time-only PI,  instruction tuning  relies on \textit{explicit instructions given at both training and inference time}.

In essence, these approaches target strong performance across a collection of distinct tasks. Prefix tuning and soft-prompt optimization use \textit{learned}, task-specific tokens while keeping the base model fixed, whereas instruction tuning relies on \textit{given}, task-specific natural-language instructions and updates the model parameters accordingly. Crucially, in all cases the prepended tokens are \textit{task-specific} and explicitly signal the \textit{known} task to the model; in contrast, in our setting the PI tokens implicitly characterize  the \textit{latent} underlying data distribution rather than specifying a task.

\textbf{Personalization} has been a central research topic for decades, particularly in recommender systems, well predating large language models \citep{adomavicius2005toward,ricci2010introduction}. The primary input space typically consists of user–item interactions, such as explicit ratings or implicit feedback (e.g., clicks, watch time). Personalization is often improved by augmenting these signals with auxiliary information about users (e.g., demographics, preferences) and items (e.g., attributes such as genre or director). In most cases, these auxiliary user/item attributes are obtained from \textit{external} sources, although in some cases they may also be inferred (e.g. buying \texttt{diapers} regularly implies \texttt{is-father},  watching \texttt{coming-of-age movies} implies \texttt{is-teenager}, etc.).  

The modern version of this setting appears in personalizing large language model (LLM) based dialogue generation to individual users \citep{ait2023power}. In such systems, user characteristics and preferences are recorded or inferred from historical interactions and retrieved, often from a memory module, to augment the model's context at generation time \citep{chu2018learning,li2025persona,jiang2025personamem}.

\textbf{Conditional generation} can be viewed as an instance of personalization for generative models, where LLMs and image generators produce outputs conditioned on auxiliary information such as user  preferences, instruction, or retrieved context \citep{mirza2014cgan,keskar2019ctrl,ramesh2021zero,rombach2022high}.

Similar to instruction tuning and personalization, the conditioning input is external and given at both training and inference time. Conditional generation focuses on instance-level conditioning, akin to personalization,  rather than distinct tasks as in instruction tuning.

\textbf{Learning with auxiliary data} is one of the most widely used strategies for improving model performance. At its core, this paradigm seeks to incorporate information beyond the original input space by leveraging additional data sources. A prominent recent trend is multi-modal fusion, where heterogeneous modalities are combined to enhance performance. Representative examples include time series forecasting augmented with textual annotations \citep{lim2021tft, zhang2025does}, joint text–image representation learning \citep{radford2021clip}, and vision–language pretraining frameworks \citep{kim2021vilt,li2023blip2}.

Across these paradigms, a unifying theme is the use of additional information beyond the raw input to improve learning. However, they differ fundamentally in the nature, availability, and role of this auxiliary information. In particular, none of the aforementioned settings is designed for tabular foundation models, nor do they explicitly operate under a \textit{train-time-only} information regime. In contrast, our framework is grounded in the LUPI paradigm, where privileged information is available exclusively during training and is not required at inference time. This distinction is critical: rather than serving as persistent conditioning input (as in personalization, conditional generation, or auxiliary-data learning), or as task-defining signals shared across train and test (as in instruction tuning and prefix-based methods), our PI acts as a transient learning signal that accelerates learning and improves generalization of TFMs. Our auxiliary information, namely the designed privileged information comprising dataset meta-features and generator program specifications, distinctly encodes  latent structure in the data distribution, enabling more efficient learning of tabular predictors under a foundation model paradigm.

%% file: TAB/related.tex
\begin{table*}[!h]
\centering
\caption{Comparison of learning paradigms by input source, task structure, and task--input assignment.
Our setting uniquely differs from related lines of work in that 
(1) the additional input, called privileged information (PI) in our case, is not sourced externally but represents \textbf{internal} distribution characteristics derived empirically (meta-features) or theoretically (generator program); (2) the latter, generator-PI is \textbf{train-time-only} while other settings use external input at both training and inference;
(3) PI represents \textbf{given}, and not learned, prefix tokens;
(4) PI is employed at instance level \textbf{per dataset}, rather than being task-specific;
(5) PI enriches the training data in the context by representing \textbf{distinct yet latent} distributions as opposed to different but known tasks.
}
\vspace{-0.05in}
\centering
\begin{tabular}{llcccc} 
\toprule
\hspace{1.75cm}/ \textbf{Setting}
& \textbf{Extra}
& \textbf{Train-} 
& \textbf{Given vs.}
& \textbf{Instance- vs. }
& \textbf{Task--Input} \\ 
\textbf{Problem} 
& \textbf{Input}
&  \textbf{only?}
& \textbf{Learned}
& \textbf{Task-specific}
& \textbf{Assignment} \\ 
\midrule
Prefix-tuning
& N/A 
& No
& Learned
& Task
& Known \\

Soft-prompt opt.
& N/A 
& No
& Learned
& Task
& Known \\

Instruction tuning
& External
& No
& Given 
& Task
& Known \\

Personalization
& External
& No
& Given
& Instance
& N/A \\

Conditional generation
& External
& No
& Given 
& Instance
& N/A \\

Prediction with aux. data
& External
& No
& Given
& Instance
& N/A \\

\midrule
{[Privileged TFMs]}
& Internal
& Yes
& Given 
& Instance
& Latent \\
\bottomrule
\end{tabular}
\label{tab:relatedcomparison}
\end{table*}

%% file: 04theory.tex
Let the prior define a space of hypotheses $\Psi$ on the relationship of a set of features to the output label. Each hypothesis $\psi \in \Psi$ can be seen as a mechanism that generates a data
distribution from which we can draw samples forming a dataset $\mathcal{D}$. 
Let $\psi \sim P_\Psi$ denote the parameters of a hypothesis sampled from the data prior (generator).
Then, synthetic datasets are sampled as
\[
\mathcal{D} = \{(\bx_i, y_i)\}_{i=1}^n \sim P(\cdot \mid \psi)\;.
\]

A tabular foundation model, denoted $h_\theta$ with parameters $\theta$, is a meta-learner that is trained on numerous datasets to learn the mapping 
\[
h_\theta : (\mathcal{D}_{\text{context}}, \bx_{\text{query}}) \mapsto {y}_{\text{query}} \;.
\]

Let $\mathcal{H}$ denote the hypothesis class induced by a fixed transformer
architecture (fixed depth, width, attention heads, and context length):
\[
\mathcal{H} = \{ h_\theta(\mathcal{D}, \bx): \theta \in \Theta \}\;.
\]

The expected meta-risk is written as
\[
R(h) =
\mathbb{E}_{\psi \sim P_\Psi}
\mathbb{E}_{\mathcal{D} \sim P(\cdot \mid \psi)}
\mathbb{E}_{(\bx,y) \sim \mathcal{D}}
\left[ \ell(h(\mathcal{D}, \bx), y) \right]\;.
\]

Let $h^*$ be the best possible decision function (in terms of generalization error) in $\mathcal{H}\textbf{}$, and let
\[
R^* = \inf_{h \in \mathcal{H}} R(h) = R(h^*)\;.
\]

Let $\mathcal{P}^*$ denote the privileged information (PI) associated with a data set constructed by the teacher.
Define the augmented hypothesis class
\[
\mathcal{H}_{\text{PI}} =
\{ h_\theta(\mathcal{P}^*, \mathcal{D}, \bx) : \theta \in \Theta_{\text{PI}} \} \;,
\]

and let 
\[
R_{\text{PI}}^* = \inf_{h \in \mathcal{H}_{\text{PI}}} R(h) \;.
\]

\subsection{PI-Driven  Representational Advantage}
\label{ssec:population}

In this section, we establish the theoretical potential of our approach, 
by characterizing the conditions under which privileged information (PI) narrows the population-level approximation gap between a standard learner and a PI-augmented hypothesis class.

\subsubsection{Meta-features based Privileged Information}

Let $\mathcal{P}^*=\bs(\mathcal{D})$ denote a set of meta-features comprising summary statistics computed
from the context dataset $\mathcal{D}$ (e.g., empirical mean, higher moments, etc.).

Although $\bs(\mathcal{D})$ is a deterministic function of $\mathcal{D}$,
the class $\mathcal{H}$ may not contain functions that reliably compute 
or represent $\bs(\mathcal{D})$ internally under fixed architectural constraints
(e.g., bounded depth, width, and attention heads).

In such cases,
$\mathcal{H} \subsetneq \mathcal{H}_{\text{PI}}$, which implies
$R_{\text{PI}}^* \le R^*$.

Thus, summary statistics can reduce achievable asymptotic meta-risk
relative to the fixed architecture, even though they do not introduce
new information in an information-theoretic sense.

\vspace{0.05in}
\begin{theorem}[Model-Relative Risk Reduction via Meta-PI under Architectural Constraints]
\label{thm:meta}
Let us define the set of \textit{useful} teachers as 
\begin{align}
\label{def:usefulness}
  {\Phi}_u = & \{\phi \in \Phi \;|\; \ell_{(\mathcal{P}^*,\mathcal{D})}\!\left(\phi(\mathcal{P}^*, \mathcal{D},  \bx)\right) \approx \ell_{\mathcal{D}}\!\left(h^*(\mathcal{D},  \bx)\right) \tag{1}\\  
    & \text{ for any } (\mathcal{P}^*, \mathcal{D},  \bx) \text{ with non-zero probability.} \} \nonumber 
\end{align}
In words, useful teacher (correcting) functions can, using the additional privileged information,  emulate the oracle corrections of the best-in-class learner for typical inputs, and therefore have an approximation error smaller than that of the learner.

Assume:

(A1) ${\Phi}_u \neq \emptyset$, i.e., 
there exists a useful teacher(s).

(A2) For each useful teacher $\phi \in \Phi_u$, there exists a decision function $h\in \mathcal{H}$ such that, for any  $(\mathcal{P}^*, \mathcal{D},  \bx)$ with non-zero probability, 
$$\ell_{(\mathcal{P}^*, \mathcal{D})}\!\left(\phi(\mathcal{P}^*, \mathcal{D},  \bx)\right) \geq \ell_{\mathcal{D}}\!\left(h(\mathcal{D},  \bx)\right) \;.$$ 
That is, the learner's hypothesis class can implement the teacher space, where every useful teacher  corresponds to a realizable learner that matches the teacher's predictions on the non-privileged inputs. 
Informally, useful teachers are not ``too good'' for the learner class.

(A3) There exists $\epsilon > 0$ such that for all
$h \in \mathcal{H}$,
\[
\mathbb{E}_{\mathcal{D}  \sim P(\cdot \mid \psi)}
\| \bz_h(\mathcal{D}) - \bs(\mathcal{D}) \|^2
\ge \epsilon,
\]
where $\bz_h$ denotes the implicit dataset-level representation computed by $h$.
That is, the learner class cannot approximate the meta-features arbitrarily well under fixed architecture constraints.

Then, under assumptions (A1)--(A3),
\[
R_{\text{PI}}^* \le R^* - c \epsilon
\]
for some constant $c > 0$ that depends on the sensitivity of the loss to the meta-features.
\end{theorem}

\begin{proof}

We show that the PI-augmented class achieves strictly lower population risk
by separating (i) realizability of useful PI-dependent predictors and
(ii) irreducible representation error in $\mathcal{H}$.

\vspace{0.5em}
\noindent
\textbf{Step 1: Existence of a useful PI-dependent predictor.}

By (A1), there exists a useful teacher $\phi \in \Phi_u$ such that its
risk is comparable to the best-in-class non-PI predictor:
\[
R(\phi) \approx R(h^*).
\]

Let $\phi$ act as a reference PI-dependent predictor.

\vspace{0.5em}
\noindent
\textbf{Step 2: Alignment with the learner class.}

By (A2), for this $\phi \in \Phi_u$, there exists
$h_\phi \in \mathcal{H}$ such that for all
$(\mathcal{P}^*, \mathcal{D}, \bx)$,
\[
\ell_{(\mathcal{P}^*, \mathcal{D})}
\big(\phi(\mathcal{P}^*, \mathcal{D}, \bx)\big)
\;\ge\;
\ell_{\mathcal{D}}\big(h_\phi(\mathcal{D}, \bx)\big).
\]

Taking expectations over the data distribution yields
\[
R(\phi) \ge R(h_\phi).
\]

Hence, every useful PI-dependent predictor has a counterpart in
$\mathcal{H}$ that is not worse when restricted to non-privileged inputs.

{This step ensures that any advantage of $\phi$ is not due to
a strictly richer hypothesis class, but must arise from how PI modifies
the representation available during learning.}

\vspace{0.5em}
\noindent
\textbf{Step 3: Representation bottleneck in $\mathcal{H}$.}

Each $h \in \mathcal{H}$ induces a dataset representation $\bz_h(\mathcal{D})$.
By (A3),
\[
\mathbb{E}_{\mathcal{D}}
\|\bz_h(\mathcal{D}) - \bs(\mathcal{D})\|^2 \ge \epsilon.
\]

That is, no $h \in \mathcal{H}$ can recover the meta-features
$\bs(\mathcal{D})$ arbitrarily well under the fixed architecture.
Thus, predictors in $\mathcal{H}$ must act on a distorted version of the
input-relevant meta-features.

\vspace{0.5em}
\noindent
\textbf{Step 4: Risk gap induced by representation distortion.}

Because useful predictors in $\Phi_u$ depend non-trivially on
$\bs(\mathcal{D})$, there exists $c > 0$ such that for any $h \in \mathcal{H}$,
\[
R(h) \ge R(\phi) + c \cdot
\mathbb{E}_{\mathcal{D}}
\|\bz_h(\mathcal{D}) - \bs(\mathcal{D})\|^2.
\]

Using (A3), this implies
\[
R(h) \ge R(\phi) + c \epsilon,
\quad \forall h \in \mathcal{H}.
\]

Taking the infimum over $\mathcal{H}$ yields
\[
R^* \ge R(\phi) + c \epsilon. \tag{1}
\]

\vspace{0.5em}
\noindent
\textbf{Step 5: PI alleviates the representation bottleneck.}

In $\mathcal{H}_{\text{PI}}$, the predictor has direct access to
$\mathcal{P}^* = \bs(\mathcal{D})$, eliminating the representation gap.
Hence there exists $h_{\text{PI}} \in \mathcal{H}_{\text{PI}}$ such that
\[
R(h_{\text{PI}}) \approx R(\phi),
\quad \Rightarrow \quad
R_{\text{PI}}^* \le R(\phi). \tag{2}
\]

\vspace{0.5em}
\noindent
\textbf{Step 6: Combining the bounds.}

Combining inequalities (1) and (2), we get 
\[
R_{\text{PI}}^* \le R(\phi) \le R^* - c \epsilon,
\]
which yields
\[
\boxed{R_{\text{PI}}^* \le R^* - c \epsilon} \;.
\]
\end{proof}

In short, Theorem \ref{thm:meta} shows that the augmented hypothesis class can achieve reduced expected (meta-)risk, that is $R_{\text{PI}}^* < R^*$, assuming \textbf{(A1)} Useful teachers exist 
($\Phi_u \neq \emptyset$), 
\textbf{(A2)} Learner class can mimic any useful teacher ($\mathcal{H}$ can realize $\Phi_u$), and 
and 
\textbf{(A3)} Learner class cannot recover meta-features exactly ($\bz_h(\mathcal{D})$ differs from $\bs(\mathcal{D})$
 by at least $\epsilon$ in expectation).

\subsubsection{Generator based 
 Privileged Information}

Next we move to the true LUPI case:
generator program available during pretraining (i.e. meta-training) but not at inference time.
This approach extends the utility of the meta-features as PI, as it enables information-theoretic gains in the asymptotic limit, beyond architecture-relative improvement.

\vspace{0.05in}
\begin{theorem}[Information-Theoretic Risk Reduction via Generator based PI]
\label{thm:generator}

Assume: (A1) and (A2) in Theorem \ref{thm:meta}, as well as

(A3) The conditional entropy $H(\psi \mid \mathcal{D})$,  which measures the remaining uncertainty about the generator (specification and parameter values)
after observing the dataset, is positive\footnote{For continuous $\psi$, this condition corresponds to the assumption
$\mathrm{Var}(\psi \mid \mathcal{D}) > 0$ {with positive probability}.}, that is,
\[
H(\psi \mid \mathcal{D})
=
\mathbb{E}_{\mathcal{D}}
\left[
- \int p(\psi \mid \mathcal{D})
\log p(\psi \mid \mathcal{D})
\, d\psi
\right] > 0 \;.
\]

Then, under assumptions (A1)–(A3), 
\[
R_{\text{PI}}^* < R^*.
\]

\end{theorem}

\begin{proof}

We show that access to generator $\mathcal{P}^*$ during training induces a hypothesis that encodes a more accurate inference of the latent generator $\psi$, leading to strictly lower population risk at test time, even though $\mathcal{P}^*$ is not available during inference.

\vspace{0.5em}
\noindent
\textbf{Step 1: Bayes-optimal prediction without PI.}

Given a dataset $\mathcal{D}$, any predictor without PI must act under
the posterior $p(\psi \mid \mathcal{D})$.
Define the Bayes-optimal decision rule
\[
h^{\mathrm{Bayes}}(\mathcal{D}, \bx)
:=
\arg\min_{\widehat{y}}
\mathbb{E}_{\psi \sim p(\psi \mid \mathcal{D})}
\;\mathbb{E}_{y \sim p(\cdot \mid \bx, \psi)}
\big[\ell(\widehat{y}, y)\big].
\]

By (A3), $H(\psi \mid \mathcal{D}) > 0$, so the posterior is non-degenerate
with positive probability. Hence $h^{\mathrm{Bayes}}$ must average over
multiple plausible generators, and therefore incurs strictly higher
risk than the oracle predictor that knows $\psi$. Then,
\[
R(h^\psi) < R(h^{\mathrm{Bayes}}),
\quad \text{and} \quad
R^* \ge R(h^{\mathrm{Bayes}})\;, \tag{1}
\]
where
\[
h^{\psi}(\mathcal{D}, \bx)
:=
\arg\min_{\widehat{y}}
\mathbb{E}_{y \sim p(\cdot \mid \bx, \psi)}
\big[\ell(\widehat{y}, y)\big].
\]

\vspace{0.5em}
\noindent
\textbf{Step 2: PI enables learning of generator-dependent predictors.}

Privileged information $\mathcal{P}^*$ provides additional information
about $\psi$ at training time. By (A1), there exists a useful teacher
$\phi \in \Phi_u$ that exploits $\mathcal{P}^*$ to approximate
oracle corrections associated with $\psi$.

Training with $\mathcal{P}^*$ enables the learner to construct a predictor
that approximates the mapping $\mathcal{D} \mapsto \psi$ and thereby
approximates $h^{\psi}$.
This induces a parameterization within $\mathcal{H}_{\text{PI}}$ that
approximates $h^\psi$ more closely than any predictor trained without PI:
\[
R(h_{\text{PI}}) \approx R(h^{\psi}),
\quad \Rightarrow \quad
R_{\text{PI}}^* \le R(h^{\psi}). \tag{2}
\]

\vspace{0.5em}
\noindent
\textbf{Step 3: Elimination of hypothesis-class confounding.}

By (A2), for the useful teacher $\phi \in \Phi_u$, there exists
$h_\phi \in \mathcal{H}$ such that its behavior over $(\mathcal{D}, \bx)$
is compatible with the base hypothesis class.

Thus, any improvement induced by $\phi$ cannot be attributed to a
strict increase in the expressivity of $\mathcal{H}$, but must instead
arise from improved inference of the latent generator $\psi$.

\vspace{0.5em}
\noindent
\textbf{Step 4: Combining the bounds.}

Combining inequalities in (1) and (2) above, we get 
\[
R_{\text{PI}}^* \le R(h^{\psi})
< R(h^{\mathrm{Bayes}})
\le R^*,
\]
which implies
\[
\boxed{
R_{\text{PI}}^* < R^*} \;.
\]

\end{proof}

In words, Theorem \ref{thm:generator} states that the generator based PI can reduce asymptotic Bayes risk
assuming \textbf{(A1)} the teacher space encodes the oracle loss surface, \textbf{(A2)} the teacher   space is within the representational limits of the learner class, and 
\textbf{(A3)} the generator prior and parameters are not identifiable from the observed dataset (i.e. multiple generators remain plausible given the dataset).

\textbf{Remarks:~}
In summary, (A1) and (A2) capture the usefulness and realizability of the privileged information: (A1) asserts that the teacher equipped with PI can emulate oracle corrections, while (A2) ensures that the learner can exploit the guidance during training. 
For example, summary statistics of a dataset, such as mean or variance, can aid with outlier detection. Similarly, knowing the task type (e.g., outlier detection under a GMM prior) via the generator parameters can allow the model to adapt its internal representations or weight updates according to the prior-specified distribution.

In both Meta-PI and Generator-PI cases, a key assumption is (A3), that \textbf{privileged information cannot be recovered from the context alone}, which is what allows it to strictly reduce risk. 
For \textbf{meta-features} as PI, (A3) reflects a \textbf{model-limited} scenario: even though the information is in principle computable from the dataset, the model cannot extract it efficiently due to architectural or in-context capacity constraints. 
For \textbf{generator program} as PI, (A3) reflects a \textbf{data-limited} scenario: the privileged parameters cannot be inferred from a limited dataset context, so they provide information that is fundamentally inaccessible without teacher guidance.

Theorems 3.1 and 3.2 characterize a \textit{population-level} approximation gap between the learner and the PI-augmented hypothesis class, outlining conditions under which the teacher provides a \textbf{strictly lower achievable risk} due to enhanced representational capacity.

In the following, we consider a \textit{finite-sample} convergence analysis, where we study learning rates and estimation effects,  identifying conditions under which the teacher enables \textbf{faster convergence} to this improved risk.

\input{031convergence}

%% file: 031convergence.tex
\subsection{Finite-Sample Learning and Convergence}
\label{ssec:finite}

Next, we analyze the empirical convergence properties of our approach in finite-data regimes, examining the  rate at which the model realizes the benefits of PI in practice.

First, assume
that the learner $h\in \mathcal{H}$ may learn the true function $t$ at a slow rate $\alpha_h=1/2$,
\begin{equation}
\label{eq:lonestudent}
    R(h) - R(t) \le \mathcal{O} \bigg({\frac{|\mathcal{H}|_{\text{VC}}}{\sqrt{n}}}\bigg) + \epsilon_h \;,
\end{equation}

where $\mathcal{O}(\cdot)$ term is the estimation error and $\epsilon_h$ is the approximation error of the learner
class $\mathcal{H}$ with respect to the target function $t$.

Second, assume that a teacher $\phi\in \Phi$ can learn at a faster rate $\beta > \alpha_t$ than the learner; 
\begin{equation}
\label{eq:teacher}
R(\phi) - R(t) \le \mathcal{O} \bigg({\frac{|{\Phi}|_{\text{VC}}}{n^\beta}}\bigg) + \epsilon_\phi \;,
\end{equation}
where
$\epsilon_\phi$ is the approximation error of the teacher function 
class ${\Phi}$ with respect to the target  $t$.

Finally,
assume that the learner can match the teacher (correcting) function at a rate $\alpha_\phi$ 
\begin{equation}
\label{eq:matchingstudent}
R(h) - R(\phi) \le \mathcal{O} \bigg(\frac{|\mathcal{H}|_{\text{VC}}}{n^{\alpha_\phi}}\bigg) + \epsilon_s \;,
\end{equation}
where $\frac{1}{2}\leq \alpha_\phi \leq 1$ and $\epsilon_s$ (`s' for student; learner is called as student when mimicking the teacher) denotes the approximation error of the learner
class $\mathcal{H}$ with respect to the teacher function $\phi\in \Phi_u$.

Summing the inequalities in \eqref{eq:teacher} and \eqref{eq:matchingstudent}
provides an alternative upper bound to the LHS in \eqref{eq:lonestudent}:
\begin{equation}
\label{eq:guidedstudent}
    R(h) - R(t) \le \mathcal{O} \bigg({\frac{|\mathcal{H}|_{\text{VC}} + |{\Phi}|_{\text{VC}}}{n^{\min\{\beta,\alpha_\phi\}}}}\bigg) + \epsilon_\phi + \epsilon_s \;.
\end{equation}

Benefits of learning with a teacher arise, i.e., RHS of \eqref{eq:guidedstudent} becomes lower than the RHS of \eqref{eq:lonestudent}, when:  

\begin{itemize}
    \item[\textbf{(1)}] the approximation error of the  teacher is smaller than that of the learner ($\epsilon_\phi \leq \epsilon_h$), hence the teacher is useful (A1); 
    \item[\textbf{(2)}] the learner can match the teacher ($\epsilon_s \approx 0$) with rate $\alpha_\phi>1/2$ such that $\min\{\beta,\alpha_\phi\} > 1/2$ (A2); and  \item[\textbf{(3)}] the capacity of the teacher class ($|{\Phi}|_{\text{VC}}$) is small. 
\end{itemize}

%% file: 05extraexp.tex
\subsection{Experiment Setup }
\label{app:setup}

\textbf{Dataset Description:} 
To evaluate different models on {in-distribution} datasets, we generate data from five distinct priors. For each prior, we sample 250 datasets by uniformly drawing the dimensionality from $[2,100]$ and randomly sampling other prior-specific variables, such as the number of GMM components, SCM DAG widths and depths, and the dimensions perturbed to create outliers. Since both the 2-layer and 10-layer models are trained with a maximum context length of 5{,}000, we restrict all evaluation datasets to have at most 5{,}000 context points. Outliers constitute $5\%$--$20\%$ of the test set. For the 2-layer models and their variants, we use the GMM subset for in-distribution evaluation. For the 10-layer models, we evaluate on the full mixed-prior test set.

For {out-of-distribution} evaluation, we evaluate the 10-layer models on ADBench \cite{han2022adbench}, a widely used benchmark suite for tabular outlier detection that aggregates diverse real-world datasets across multiple domains and anomaly types.

\textbf{Training Configurations:~} We train the 2-layer model for 1,500 epochs on sequences of length 5,000, using 1,000 steps per epoch, batch size 2, and 4 v100-32G GPUs. The model uses a 256-dimensional embedding, 2 attention heads, hidden dimension 512, and 500 learned routing tokens, with the last layer also using routing attention. Optimization is performed with a peak learning rate of $10^{-4}$ with AdamW optimizer and cosine learning rate decay. For the enlarged model, the embedding dimension is increased from 256 to 512, the number of attention heads from 2 to 8, the hidden dimension from 512 to 1024, and the depth from 2 layers to 10 layers. This enlarged configuration substantially increases model capacity, enabling stronger sequence/context modeling  while keeping the overall training protocol unchanged (we train on 4 L40 GPUs).

\subsection{Comparing PI vs. No-PI TFMs across Training Epochs}
\label{app:pairedtests}

We compare \method, Silver Enc., and No-PI when they all reach the same loss as shown in Figure \ref{fig:main}.  Specifically, \method (transferred) and Silver Enc. (no transfer) reach No-PI's loss at convergence @1500 epochs respectively @900 and @750 epochs. 
Figure \ref{fig:ptest_10layers} presents the corresponding permutation test results, showing that Silver Enc@750 outperforms No-PI@1500, and that  \textsc{PIQL}@900 significantly outperforms both. 

Similarly, we compare  Gold@250, Silver@450, Bronze@1000, and No-PI@1500 epochs for 2-layer architecture when they all reach the same loss as shown in Figure \ref{fig:prelim}. 
Figure \ref{fig:ptest_2layers}
presents the corresponding permutation test results, showing that all privileged models significantly outperform the standard base model No-PI@1500 even at earlier epochs, with earlier gains observed as PI quality increases from bronze to gold. Moreover, generator-PI based Silver and Gold models are comparable at @450 and @250 epochs, respectively, both of which outperform Bronze @1000 epochs. 

\begin{figure}[!ht]
    \centering
    \begin{subfigure}[t]{0.48\linewidth}
        \centering
        \includegraphics[width=\linewidth]{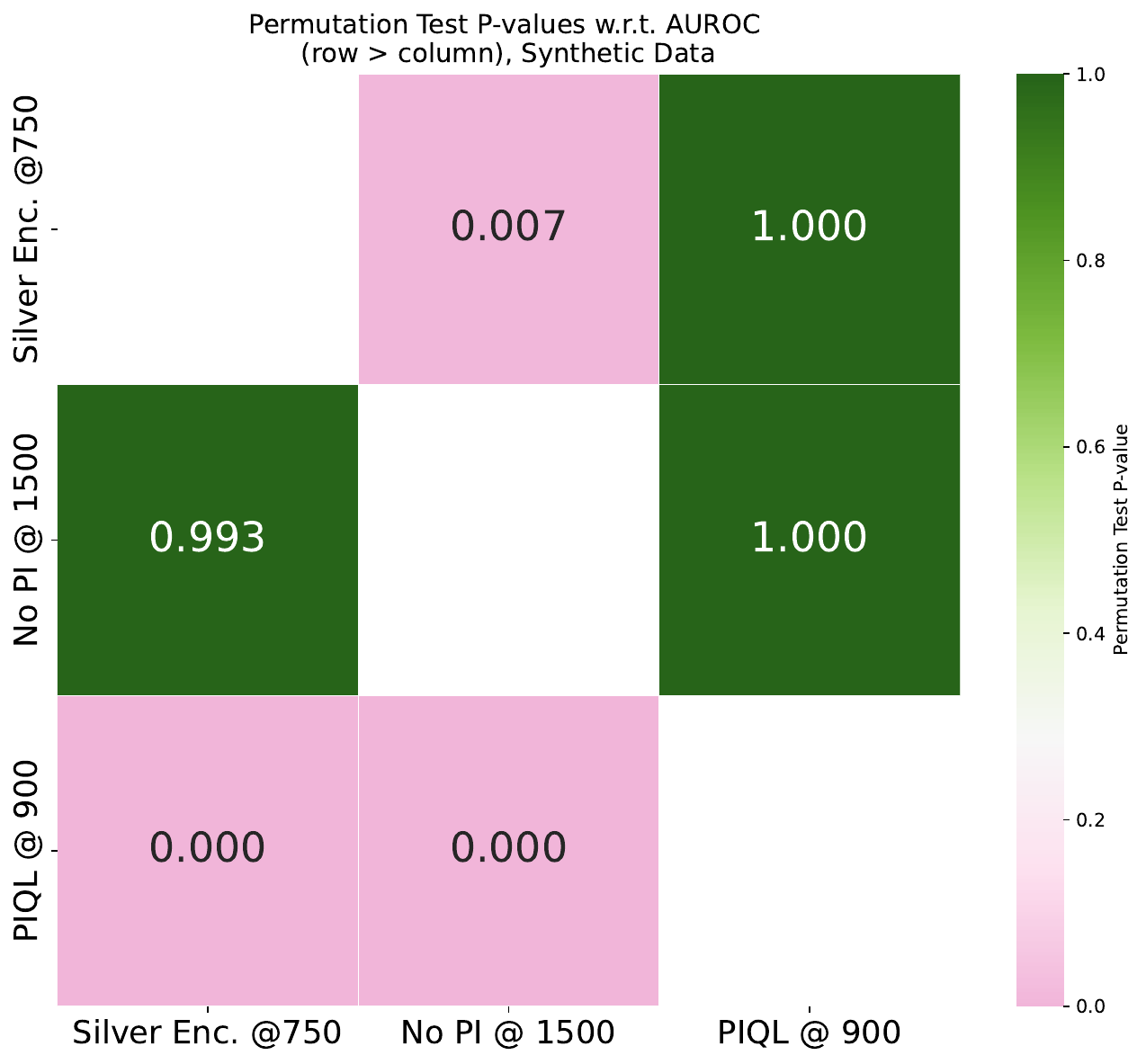}
        \caption{Paired permutation test of \textsc{PIQL} @900, Silver Enc. @750, and No-PI @1500 epochs for the 10-layer architecture (as shown in Fig.~\ref{fig:main}), taken at varying epochs as they reach the same loss. \textsc{PIQL}@900 significantly outperforms both Silver Enc@750 and No-PI@1500.}
        \label{fig:ptest_10layers}
    \end{subfigure}
    \hfill
    \begin{subfigure}[t]{0.48\linewidth}
        \centering
        \includegraphics[width=\linewidth]{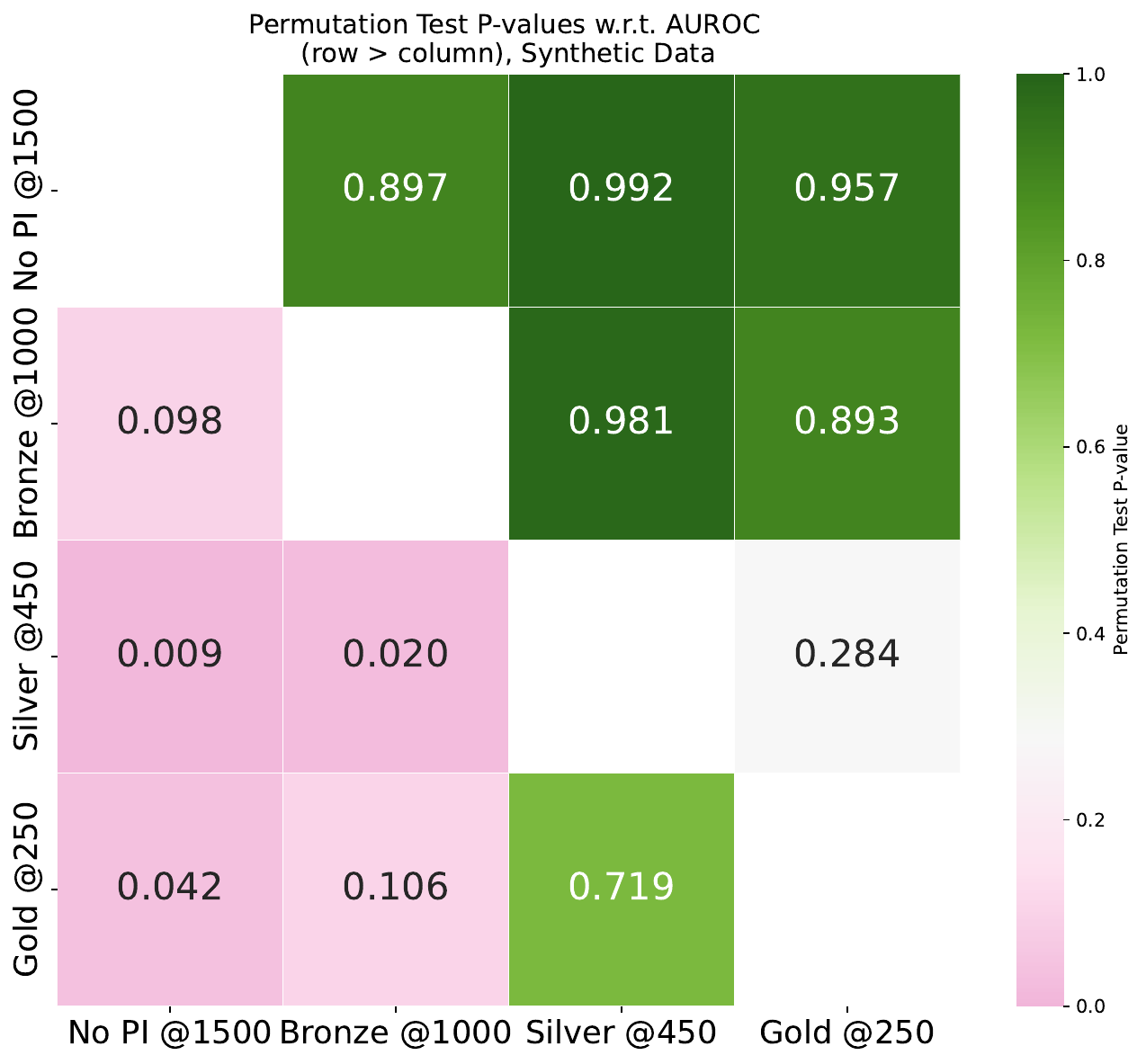}
        \caption{Paired permutation test of Gold @250, Silver @450, Bronze @1000, and No-PI @1500 epochs for the 2-layer architecture (as shown in Fig.~\ref{fig:prelim}), taken at varying epochs as they reach the same loss. All privileged models significantly outperform
  No-PI@1500 at much earlier epochs, improving from bronze to gold. }
        \label{fig:ptest_2layers}
    \end{subfigure}
    \caption{Pairwise permutation test results for the 10-layer and 2-layer models. }
    \label{fig:ptest_combined}
\end{figure}

\subsection{Individual Dataset Results in ADBench}
\label{ssec:adb}

Table \ref{tab:appendix_fulladbench_auroc} and Table \ref{tab:appendix_fulladbench_aucpr} show individual AUROC and AUPRC results on ADBench. Table \ref{tab:adbench_pval} shows the permutation test results between \method and the base model at varying epochs. \method shows statistically similar results to the standard base model at earlier epochs, exhibiting faster learning rate.

\begin{table}[h]
 \caption{Paired comparison between \method'ed TFM vs no-PI baseline across varying epochs based on $p$-value from permutation test on ADBench. $p>0.05$ shows no statistically significant difference, indicating that \method enables TFMs to match the no-PI baseline performance at earlier epochs.}
    \centering
    \vspace{0.05in}
   \begin{tabular}{lcc}
\toprule
 & \textsc{PIQL}@650 & \textsc{PIQL}@700 \\
\midrule
No-PI@920  & 0.299 & 0.084 \\
No-PI@950  & 0.497 & 0.204 \\
No-PI@1000 & 0.833 & 0.390 \\
No-PI@1050 & 0.438 & 0.119 \\
\bottomrule
\end{tabular}
\label{tab:adbench_pval}
\end{table}

\subsection{LLM-generated Meta-features for Meta-PI++}
\label{ssec:llmmeta}

For Meta-PI++, we prompt a LLM (ChatGPT) to describe 91 additional summary statistics for any dataset besides the 9 existing ones: mean, variance, skewness, kurtosis, q10, q25, q50, q75, and q90. LLM provides the additional statistics as grouped into seven categories as shown in Table \ref{tab:additional_stats}: 39 fine-grained quantile features describing the marginal distribution, 15 robust spread measures capturing variability and interval widths, 8 extreme-value statistics summarizing minima, maxima, ranges, and absolute magnitudes, 8 higher-moment statistics characterizing distributional shape beyond skewness and kurtosis, 12 tail/outlier proportion features measuring mass in extreme regions, 6 sign/zero/finite-value indicators capturing numerical and sparsity patterns, and 3 asymmetry measures quantifying skewness and mean–median imbalance. We also work with the LLM to generate the code that produces the described meta-features. 

Table \ref{tab:additional_stats} presents the list and the detailed description of the additional dataset-level statistics we use as meta-features in Meta-PI++.

\begin{table}[!t]
\hspace{-0.2in}
\centering
\caption{List of the additional LLM-generated meta-features for Meta-PI++.}
\label{tab:additional_stats}
\small
\begin{tabular}{p{0.2\linewidth}p{0.06\linewidth}p{0.65\linewidth}}
\toprule
\textbf{Category} & \textbf{Count} & \textbf{Additional features} \\
\midrule
Extra quantiles 
& 39 
& $q_{0.001}$, $q_{0.01}$, $q_{0.025}$, $q_{0.05}$, $q_{0.075}$, $q_{0.125}$, $q_{0.15}$, $q_{0.175}$, $q_{0.20}$, $q_{0.225}$, $q_{0.275}$, $q_{0.30}$, $q_{0.325}$, $q_{0.35}$, $q_{0.375}$, $q_{0.40}$, $q_{0.425}$, $q_{0.45}$, $q_{0.475}$, $q_{0.525}$, $q_{0.55}$, $q_{0.575}$, $q_{0.60}$, $q_{0.625}$, $q_{0.65}$, $q_{0.675}$, $q_{0.70}$, $q_{0.725}$, $q_{0.775}$, $q_{0.80}$, $q_{0.825}$, $q_{0.85}$, $q_{0.875}$, $q_{0.925}$, $q_{0.95}$, $q_{0.975}$, $q_{0.99}$, $q_{0.995}$, $q_{0.999}$ \\
\midrule
Robust spread / interval widths 
& 15 
& IQR, 90\% range, 80\% range, 50\% range, lower-tail width $(q_{0.25}-q_{0.10})$, upper-tail width $(q_{0.90}-q_{0.75})$, lower extreme-tail width $(q_{0.10}-q_{0.01})$, upper extreme-tail width $(q_{0.99}-q_{0.90})$, median absolute deviation, mean absolute deviation, std-to-mean ratio, IQR-to-std ratio, range-to-std ratio, $q_{0.90}/|q_{0.10}|$, $q_{0.75}/|q_{0.25}|$ \\
\midrule
Min/max/extreme statistics 
& 8 
& Minimum, maximum, range, absolute minimum, absolute maximum, mean absolute value, median absolute value, RMS \\
\midrule
Higher moments 
& 8 
& 5th standardized moment, 6th standardized moment, 2nd raw moment, 3rd raw moment, 4th raw moment, 2nd central moment, 3rd central moment, 4th central moment \\
\midrule
Tail / outlier proportions 
& 12 
& Fraction below $q_{0.01}$, below $q_{0.05}$, below $q_{0.10}$, above $q_{0.90}$, above $q_{0.95}$, above $q_{0.99}$, fraction with $|x-\mu|>\sigma$, $|x-\mu|>2\sigma$, $|x-\mu|>3\sigma$, fraction below lower IQR bound, above upper IQR bound, total IQR outlier fraction \\
\midrule
Sign / zero / finite behavior 
& 6 
& Fraction positive, fraction negative, fraction zero, fraction NaN, fraction Inf, fraction finite \\
\midrule
Shape asymmetry statistics 
& 3 
& Bowley skewness, tail asymmetry, normalized mean--median gap \\
\midrule
\textbf{Total} & \textbf{91} & Additional per-dimension statistics \\
\bottomrule
\end{tabular}
\end{table}

\input{TAB/auroc_adbench}

%% file: TAB/auroc_adbench.tex
\begin{table}
  \centering
  \caption{Individual dataset performances on ADBench w.r.t. \textbf{AUROC} ($\pm$ standard dev. over five seeds).  \textcolor{blue}{\textbf{Best}} and \textcolor{green!50!black}{\underline{Second-best}} results are highlighted. \method reaches statistically similar performance to base model (no-PI) at earlier epochs (see Table \ref{tab:adbench_pval}) }
\resizebox{0.95\textwidth}{!}{\begin{tabular}{l|llllll}
\toprule
Dataset & PIQL@650 & PIQL@700 &  No-PI@920 & No-PI@950 & No-PI@1000 & No-PI@1050\\ \midrule 
agnews & \pmv{60.26}{1.46} & \pmv{62.85}{0.82} & \pmv{63.50}{2.72} & \pmv{64.58}{2.42} & \pmv{\secondbest{65.74}}{2.02} & \pmv{\best{65.82}}{2.37} \\
20news & \pmv{57.12}{1.25} & \pmv{58.91}{2.34} & \pmv{59.50}{1.96} & \pmv{\secondbest{60.98}}{1.91} & \pmv{60.69}{1.57} & \pmv{\best{61.08}}{1.74} \\
SVHN & \pmv{58.57}{2.39} & \pmv{58.57}{3.41} & \pmv{59.76}{1.58} & \pmv{59.73}{1.46} & \pmv{\secondbest{60.24}}{1.39} & \pmv{\best{61.19}}{1.04} \\
FashionMNIST & \pmv{81.87}{4.72} & \pmv{80.93}{7.20} & \pmv{84.55}{3.41} & \pmv{84.69}{3.15} & \pmv{\secondbest{85.61}}{2.28} & \pmv{\best{86.72}}{0.92} \\
MNIST-C & \pmv{76.43}{4.71} & \pmv{77.15}{4.07} & \pmv{78.33}{1.26} & \pmv{78.35}{2.00} & \pmv{\secondbest{79.05}}{2.38} & \pmv{\best{80.43}}{1.97} \\
CIFAR10 & \pmv{64.42}{1.59} & \pmv{64.26}{2.05} & \pmv{64.78}{1.78} & \pmv{65.09}{1.72} & \pmv{\secondbest{65.52}}{1.33} & \pmv{\best{65.96}}{1.02} \\
MVTec-AD & \pmv{82.48}{0.55} & \pmv{82.59}{0.58} & \pmv{84.09}{1.06} & \pmv{84.65}{1.14} & \pmv{\best{84.91}}{0.80} & \pmv{\secondbest{84.89}}{1.02} \\
wilt & \pmv{91.79}{0.00} & \pmv{\secondbest{93.83}}{0.00} & \pmv{90.21}{0.00} & \pmv{93.70}{0.00} & \pmv{\best{94.18}}{0.00} & \pmv{91.96}{0.00} \\
mnist & \pmv{67.89}{0.00} & \pmv{81.80}{0.00} & \pmv{\secondbest{84.57}}{0.00} & \pmv{82.76}{0.00} & \pmv{\best{84.67}}{0.00} & \pmv{80.01}{0.00} \\
hepatitis & \pmv{99.89}{0.22} & \pmv{99.89}{0.22} & \pmv{\secondbest{99.92}}{0.17} & \pmv{99.90}{0.19} & \pmv{99.39}{0.84} & \pmv{\best{99.92}}{0.17} \\
yelp & \pmv{58.59}{8.74} & \pmv{58.27}{6.25} & \pmv{60.72}{6.06} & \pmv{61.34}{5.76} & \pmv{\secondbest{62.21}}{3.79} & \pmv{\best{62.81}}{3.68} \\
spambase & \pmv{80.26}{0.00} & \pmv{\best{83.44}}{0.00} & \pmv{83.05}{0.00} & \pmv{\secondbest{83.18}}{0.00} & \pmv{82.98}{0.00} & \pmv{81.74}{0.00} \\
smtp & \pmv{91.50}{2.22} & \pmv{\secondbest{94.79}}{1.36} & \pmv{91.84}{2.32} & \pmv{91.07}{3.04} & \pmv{\best{95.31}}{0.63} & \pmv{87.39}{4.10} \\
speech & \pmv{\secondbest{40.30}}{3.18} & \pmv{\best{40.79}}{2.61} & \pmv{40.02}{2.99} & \pmv{40.11}{1.90} & \pmv{40.16}{2.01} & \pmv{39.13}{1.81} \\
pima & \pmv{80.17}{1.33} & \pmv{80.29}{1.12} & \pmv{79.39}{1.44} & \pmv{79.73}{1.57} & \pmv{\secondbest{80.78}}{1.69} & \pmv{\best{81.18}}{1.08} \\
internetads & \pmv{50.92}{9.90} & \pmv{54.62}{4.28} & \pmv{54.95}{5.28} & \pmv{54.01}{6.64} & \pmv{\secondbest{55.37}}{4.73} & \pmv{\best{55.84}}{3.30} \\
stamps & \pmv{\best{98.75}}{0.55} & \pmv{98.17}{0.49} & \pmv{97.81}{0.77} & \pmv{98.33}{0.52} & \pmv{\secondbest{98.57}}{0.62} & \pmv{97.56}{0.95} \\
cardiotocography & \pmv{\best{62.79}}{0.00} & \pmv{60.18}{0.00} & \pmv{59.82}{0.00} & \pmv{\secondbest{62.62}}{0.00} & \pmv{59.58}{0.00} & \pmv{61.46}{0.00} \\
imdb & \pmv{\best{49.85}}{3.42} & \pmv{48.06}{3.26} & \pmv{48.75}{3.80} & \pmv{48.33}{3.61} & \pmv{47.94}{4.01} & \pmv{\secondbest{49.36}}{4.66} \\
aloi & \pmv{51.97}{0.46} & \pmv{51.22}{0.69} & \pmv{52.01}{0.70} & \pmv{51.92}{0.38} & \pmv{\best{53.49}}{1.16} & \pmv{\secondbest{52.03}}{0.24} \\
wpbc & \pmv{91.24}{1.94} & \pmv{89.82}{1.38} & \pmv{\secondbest{92.35}}{0.82} & \pmv{\best{93.03}}{1.18} & \pmv{90.08}{0.76} & \pmv{90.71}{1.98} \\
yeast & \pmv{47.56}{0.00} & \pmv{\best{51.58}}{0.00} & \pmv{48.47}{0.00} & \pmv{49.46}{0.00} & \pmv{\secondbest{49.48}}{0.00} & \pmv{49.02}{0.00} \\
fraud & \pmv{95.23}{0.75} & \pmv{95.33}{0.84} & \pmv{95.73}{0.56} & \pmv{\best{96.11}}{0.76} & \pmv{\secondbest{95.90}}{0.88} & \pmv{95.67}{1.12} \\
musk & \pmv{\secondbest{99.99}}{0.01} & \pmv{\best{99.99}}{0.01} & \pmv{99.96}{0.07} & \pmv{99.98}{0.02} & \pmv{99.96}{0.02} & \pmv{99.97}{0.04} \\
fault & \pmv{54.31}{0.00} & \pmv{\best{57.44}}{0.00} & \pmv{55.35}{0.00} & \pmv{\secondbest{56.90}}{0.00} & \pmv{55.50}{0.00} & \pmv{55.21}{0.00} \\
amazon & \pmv{52.69}{5.17} & \pmv{52.75}{4.77} & \pmv{55.18}{4.08} & \pmv{55.17}{4.02} & \pmv{\secondbest{55.90}}{2.97} & \pmv{\best{56.70}}{1.57} \\
campaign & \pmv{63.31}{1.39} & \pmv{\secondbest{72.99}}{0.88} & \pmv{71.51}{1.26} & \pmv{66.73}{1.33} & \pmv{72.31}{1.40} & \pmv{\best{73.29}}{1.11} \\
breastw & \pmv{\best{98.81}}{0.27} & \pmv{\secondbest{98.69}}{0.32} & \pmv{98.42}{0.58} & \pmv{97.47}{1.04} & \pmv{96.82}{0.99} & \pmv{97.23}{0.97} \\
glass & \pmv{98.15}{0.91} & \pmv{\secondbest{98.70}}{0.96} & \pmv{97.80}{0.95} & \pmv{95.29}{3.86} & \pmv{98.15}{1.30} & \pmv{\best{98.72}}{1.12} \\
thyroid & \pmv{\best{98.40}}{0.00} & \pmv{\secondbest{98.23}}{0.00} & \pmv{96.64}{0.00} & \pmv{97.40}{0.00} & \pmv{95.56}{0.00} & \pmv{96.50}{0.00} \\
http & \pmv{99.97}{0.02} & \pmv{\best{99.99}}{0.01} & \pmv{99.92}{0.13} & \pmv{99.98}{0.02} & \pmv{99.96}{0.03} & \pmv{\secondbest{99.98}}{0.02} \\
optdigits & \pmv{82.01}{0.00} & \pmv{75.13}{0.00} & \pmv{\secondbest{82.78}}{0.00} & \pmv{72.82}{0.00} & \pmv{82.54}{0.00} & \pmv{\best{85.83}}{0.00} \\
skin & \pmv{82.43}{2.82} & \pmv{\best{86.35}}{1.92} & \pmv{52.70}{3.25} & \pmv{62.11}{6.90} & \pmv{\secondbest{86.20}}{2.63} & \pmv{72.12}{4.32} \\
backdoor & \pmv{\best{61.67}}{16.45} & \pmv{\secondbest{60.50}}{9.22} & \pmv{54.49}{19.27} & \pmv{55.44}{19.54} & \pmv{57.21}{21.43} & \pmv{46.53}{9.69} \\
letter & \pmv{39.22}{0.00} & \pmv{\secondbest{40.28}}{0.00} & \pmv{\best{40.61}}{0.00} & \pmv{39.46}{0.00} & \pmv{39.00}{0.00} & \pmv{39.15}{0.00} \\
vertebral & \pmv{79.61}{2.98} & \pmv{81.20}{1.64} & \pmv{\best{82.07}}{2.69} & \pmv{\secondbest{81.74}}{2.86} & \pmv{77.37}{5.43} & \pmv{70.07}{4.12} \\
satimage-2 & \pmv{\best{99.62}}{0.00} & \pmv{99.24}{0.00} & \pmv{99.14}{0.00} & \pmv{99.08}{0.00} & \pmv{98.84}{0.00} & \pmv{\secondbest{99.24}}{0.00} \\
wdbc & \pmv{\secondbest{99.17}}{0.26} & \pmv{98.85}{0.47} & \pmv{99.11}{0.26} & \pmv{\best{99.18}}{0.20} & \pmv{98.82}{0.38} & \pmv{98.91}{0.45} \\
celeba & \pmv{\secondbest{69.40}}{0.77} & \pmv{\best{72.92}}{0.63} & \pmv{63.31}{1.18} & \pmv{66.84}{0.99} & \pmv{69.32}{1.02} & \pmv{63.96}{1.32} \\
landsat & \pmv{55.54}{0.00} & \pmv{\best{59.61}}{0.00} & \pmv{54.23}{0.00} & \pmv{49.45}{0.00} & \pmv{50.77}{0.00} & \pmv{\secondbest{58.14}}{0.00} \\
vowels & \pmv{86.59}{0.00} & \pmv{85.31}{0.00} & \pmv{82.62}{0.00} & \pmv{86.45}{0.00} & \pmv{\best{86.73}}{0.00} & \pmv{\secondbest{86.61}}{0.00} \\
shuttle & \pmv{99.38}{0.30} & \pmv{99.29}{0.40} & \pmv{99.39}{0.22} & \pmv{\best{99.60}}{0.15} & \pmv{\secondbest{99.54}}{0.22} & \pmv{99.45}{0.42} \\
wine & \pmv{99.91}{0.17} & \pmv{99.93}{0.14} & \pmv{99.92}{0.17} & \pmv{\secondbest{99.94}}{0.12} & \pmv{99.89}{0.21} & \pmv{\best{99.94}}{0.12} \\
annthyroid & \pmv{77.41}{0.00} & \pmv{81.51}{0.00} & \pmv{85.14}{0.00} & \pmv{\secondbest{85.66}}{0.00} & \pmv{\best{86.95}}{0.00} & \pmv{81.47}{0.00} \\
census & \pmv{60.92}{3.88} & \pmv{58.97}{4.10} & \pmv{61.43}{5.70} & \pmv{\best{62.53}}{3.41} & \pmv{\secondbest{62.47}}{4.34} & \pmv{61.16}{6.26} \\
waveform & \pmv{70.61}{0.00} & \pmv{71.63}{0.00} & \pmv{\secondbest{72.35}}{0.00} & \pmv{\best{72.88}}{0.00} & \pmv{71.40}{0.00} & \pmv{70.99}{0.00} \\
cardio & \pmv{\best{92.31}}{0.00} & \pmv{91.29}{0.00} & \pmv{90.97}{0.00} & \pmv{\secondbest{91.88}}{0.00} & \pmv{90.58}{0.00} & \pmv{90.89}{0.00} \\
wbc & \pmv{\secondbest{99.44}}{0.74} & \pmv{\best{99.54}}{0.61} & \pmv{99.25}{0.83} & \pmv{99.07}{0.99} & \pmv{98.97}{1.13} & \pmv{99.23}{0.91} \\
ionosphere & \pmv{98.05}{0.81} & \pmv{98.27}{0.82} & \pmv{98.34}{0.58} & \pmv{\best{98.69}}{0.45} & \pmv{98.38}{0.60} & \pmv{\secondbest{98.54}}{0.53} \\
cover & \pmv{99.60}{0.09} & \pmv{99.54}{0.11} & \pmv{99.51}{0.13} & \pmv{99.60}{0.10} & \pmv{\secondbest{99.66}}{0.10} & \pmv{\best{99.66}}{0.10} \\
pageblocks & \pmv{\secondbest{88.29}}{0.00} & \pmv{87.85}{0.00} & \pmv{86.82}{0.00} & \pmv{88.23}{0.00} & \pmv{87.50}{0.00} & \pmv{\best{88.73}}{0.00} \\
pendigits & \pmv{98.83}{0.00} & \pmv{99.07}{0.00} & \pmv{\best{99.52}}{0.00} & \pmv{99.40}{0.00} & \pmv{99.41}{0.00} & \pmv{\secondbest{99.51}}{0.00} \\
satellite & \pmv{74.96}{0.00} & \pmv{\secondbest{75.59}}{0.00} & \pmv{73.58}{0.00} & \pmv{73.31}{0.00} & \pmv{73.66}{0.00} & \pmv{\best{76.05}}{0.00} \\
magic.gamma & \pmv{85.43}{0.14} & \pmv{86.70}{0.12} & \pmv{\secondbest{86.96}}{0.21} & \pmv{\best{88.99}}{0.19} & \pmv{84.12}{0.22} & \pmv{86.64}{0.10} \\
donors & \pmv{\best{99.72}}{0.06} & \pmv{\secondbest{99.70}}{0.03} & \pmv{99.31}{0.50} & \pmv{99.61}{0.04} & \pmv{98.95}{0.93} & \pmv{99.29}{0.28} \\
lymphography & \pmv{99.98}{0.04} & \pmv{\secondbest{99.98}}{0.04} & \pmv{\best{99.98}}{0.04} & \pmv{99.93}{0.09} & \pmv{99.93}{0.09} & \pmv{99.96}{0.05} \\
mammography & \pmv{83.22}{4.02} & \pmv{\secondbest{86.31}}{1.70} & \pmv{83.69}{1.14} & \pmv{\best{88.41}}{0.59} & \pmv{65.52}{1.13} & \pmv{77.80}{1.78} \\
\bottomrule 
    \end{tabular}}  \label{tab:appendix_fulladbench_auroc}
  \end{table}

 \begin{table}
  \centering
  \caption{Individual dataset performances on ADBench w.r.t. \textbf{AUPRC} ($\pm$ standard dev. over five seeds).  \textcolor{blue}{\textbf{Best}} and \textcolor{green!50!black}{\underline{Second-best}} results are highlighted. \method reaches statistically similar performance to base model (no-PI) at earlier epochs (see Table \ref{tab:adbench_pval}).}
\resizebox{0.95\textwidth}{!}{\begin{tabular}{l|llllll}
\toprule
Dataset & PIQL@650 & PIQL@700 &  No-PI@920 & No-PI@950 & No-PI@1000 & No-PI@1050\\ \midrule 
agnews & \pmv{14.96}{0.83} & \pmv{16.41}{0.51} & \pmv{16.86}{1.90} & \pmv{17.20}{1.41} & \pmv{\best{17.64}}{1.38} & \pmv{\secondbest{17.63}}{1.56} \\
20news & \pmv{13.87}{0.53} & \pmv{14.55}{1.27} & \pmv{14.75}{0.57} & \pmv{\best{16.08}}{0.89} & \pmv{\secondbest{15.33}}{0.53} & \pmv{15.13}{0.60} \\
SVHN & \pmv{13.57}{1.14} & \pmv{13.68}{1.61} & \pmv{14.33}{0.84} & \pmv{14.24}{0.79} & \pmv{\secondbest{14.44}}{0.84} & \pmv{\best{14.85}}{0.63} \\
FashionMNIST & \pmv{44.32}{7.98} & \pmv{44.71}{11.42} & \pmv{49.68}{5.01} & \pmv{\best{50.42}}{5.66} & \pmv{48.92}{7.28} & \pmv{\secondbest{50.38}}{4.22} \\
MNIST-C & \pmv{38.40}{7.07} & \pmv{38.86}{6.82} & \pmv{\secondbest{41.87}}{2.17} & \pmv{41.64}{3.09} & \pmv{41.61}{4.97} & \pmv{\best{44.58}}{1.99} \\
CIFAR10 & \pmv{17.69}{1.03} & \pmv{17.61}{1.62} & \pmv{18.12}{1.21} & \pmv{18.36}{1.23} & \pmv{\secondbest{18.53}}{1.29} & \pmv{\best{18.84}}{0.89} \\
MVTec-AD & \pmv{79.09}{1.02} & \pmv{78.98}{1.13} & \pmv{80.01}{1.21} & \pmv{80.81}{1.24} & \pmv{\secondbest{81.00}}{0.95} & \pmv{\best{81.33}}{1.01} \\
wilt & \pmv{46.61}{0.00} & \pmv{55.84}{0.00} & \pmv{42.02}{0.00} & \pmv{\secondbest{59.61}}{0.00} & \pmv{\best{59.82}}{0.00} & \pmv{50.02}{0.00} \\
mnist & \pmv{44.25}{0.00} & \pmv{51.35}{0.00} & \pmv{\secondbest{58.19}}{0.00} & \pmv{55.11}{0.00} & \pmv{\best{58.69}}{0.00} & \pmv{56.72}{0.00} \\
hepatitis & \pmv{99.57}{0.86} & \pmv{99.57}{0.86} & \pmv{\secondbest{99.71}}{0.57} & \pmv{99.65}{0.71} & \pmv{98.82}{1.39} & \pmv{\best{99.72}}{0.57} \\
yelp & \pmv{13.61}{3.77} & \pmv{12.97}{2.17} & \pmv{\best{14.22}}{2.26} & \pmv{13.72}{2.03} & \pmv{14.01}{1.57} & \pmv{\secondbest{14.11}}{1.26} \\
spambase & \pmv{82.05}{0.00} & \pmv{\secondbest{82.15}}{0.00} & \pmv{\best{82.46}}{0.00} & \pmv{82.08}{0.00} & \pmv{81.69}{0.00} & \pmv{81.31}{0.00} \\
smtp & \pmv{\best{47.97}}{14.45} & \pmv{\secondbest{43.34}}{13.02} & \pmv{24.10}{19.03} & \pmv{37.56}{9.13} & \pmv{42.13}{11.71} & \pmv{28.82}{8.15} \\
speech & \pmv{2.73}{0.30} & \pmv{2.73}{0.21} & \pmv{2.74}{0.29} & \pmv{\secondbest{2.78}}{0.30} & \pmv{\best{2.82}}{0.34} & \pmv{2.69}{0.24} \\
pima & \pmv{79.29}{1.94} & \pmv{\secondbest{79.71}}{1.85} & \pmv{78.59}{2.23} & \pmv{78.98}{1.93} & \pmv{79.62}{2.21} & \pmv{\best{80.01}}{1.94} \\
internetads & \pmv{38.04}{8.26} & \pmv{\secondbest{40.88}}{7.68} & \pmv{40.24}{7.24} & \pmv{39.38}{6.67} & \pmv{40.53}{6.37} & \pmv{\best{41.23}}{5.88} \\
stamps & \pmv{\best{90.56}}{5.29} & \pmv{87.77}{4.82} & \pmv{86.90}{4.72} & \pmv{89.04}{4.64} & \pmv{\secondbest{89.79}}{5.28} & \pmv{88.80}{4.41} \\
cardiotocography & \pmv{\best{57.56}}{0.00} & \pmv{55.09}{0.00} & \pmv{55.03}{0.00} & \pmv{\secondbest{56.98}}{0.00} & \pmv{54.42}{0.00} & \pmv{56.37}{0.00} \\
imdb & \pmv{\best{9.46}}{1.12} & \pmv{8.94}{0.76} & \pmv{9.16}{1.06} & \pmv{9.07}{0.97} & \pmv{8.97}{0.95} & \pmv{\secondbest{9.30}}{1.29} \\
aloi & \pmv{\best{6.68}}{0.24} & \pmv{6.34}{0.02} & \pmv{6.45}{0.07} & \pmv{6.36}{0.06} & \pmv{\secondbest{6.57}}{0.13} & \pmv{6.33}{0.03} \\
wpbc & \pmv{80.84}{4.45} & \pmv{79.70}{3.98} & \pmv{\secondbest{83.11}}{3.74} & \pmv{\best{83.37}}{4.06} & \pmv{81.31}{3.55} & \pmv{81.22}{4.34} \\
yeast & \pmv{50.52}{0.00} & \pmv{\best{52.69}}{0.00} & \pmv{51.56}{0.00} & \pmv{\secondbest{52.24}}{0.00} & \pmv{52.13}{0.00} & \pmv{51.68}{0.00} \\
fraud & \pmv{42.27}{11.92} & \pmv{37.40}{10.47} & \pmv{\secondbest{45.84}}{12.15} & \pmv{41.85}{13.52} & \pmv{42.54}{10.95} & \pmv{\best{48.73}}{14.70} \\
musk & \pmv{\best{99.88}}{0.18} & \pmv{\secondbest{99.77}}{0.24} & \pmv{99.45}{0.89} & \pmv{99.65}{0.34} & \pmv{99.50}{0.30} & \pmv{99.62}{0.59} \\
fault & \pmv{61.22}{0.00} & \pmv{\secondbest{62.34}}{0.00} & \pmv{61.30}{0.00} & \pmv{\best{62.53}}{0.00} & \pmv{61.04}{0.00} & \pmv{61.64}{0.00} \\
amazon & \pmv{10.38}{1.46} & \pmv{10.32}{1.23} & \pmv{10.97}{1.10} & \pmv{10.87}{1.06} & \pmv{\secondbest{11.00}}{0.77} & \pmv{\best{11.31}}{0.40} \\
campaign & \pmv{38.62}{1.02} & \pmv{43.33}{0.87} & \pmv{43.46}{1.03} & \pmv{40.21}{0.94} & \pmv{\secondbest{43.52}}{1.51} & \pmv{\best{44.73}}{0.90} \\
breastw & \pmv{\best{98.45}}{0.30} & \pmv{\secondbest{98.09}}{0.49} & \pmv{97.60}{0.93} & \pmv{96.02}{1.27} & \pmv{94.94}{1.39} & \pmv{95.27}{0.96} \\
glass & \pmv{77.35}{10.46} & \pmv{\secondbest{81.92}}{9.60} & \pmv{73.11}{12.12} & \pmv{71.94}{13.48} & \pmv{76.07}{13.64} & \pmv{\best{82.38}}{11.50} \\
thyroid & \pmv{\best{72.20}}{0.00} & \pmv{\secondbest{69.61}}{0.00} & \pmv{57.88}{0.00} & \pmv{66.19}{0.00} & \pmv{52.60}{0.00} & \pmv{64.09}{0.00} \\
http & \pmv{97.63}{2.74} & \pmv{\best{99.16}}{1.21} & \pmv{\secondbest{97.76}}{2.35} & \pmv{97.47}{2.85} & \pmv{94.19}{4.19} & \pmv{97.33}{2.80} \\
optdigits & \pmv{15.23}{0.00} & \pmv{11.88}{0.00} & \pmv{16.20}{0.00} & \pmv{11.78}{0.00} & \pmv{\secondbest{16.25}}{0.00} & \pmv{\best{19.30}}{0.00} \\
skin & \pmv{64.06}{4.28} & \pmv{\secondbest{76.97}}{2.24} & \pmv{43.78}{2.67} & \pmv{46.52}{8.07} & \pmv{\best{77.78}}{2.91} & \pmv{59.51}{3.77} \\
backdoor & \pmv{\secondbest{9.84}}{3.74} & \pmv{9.55}{1.40} & \pmv{8.10}{4.26} & \pmv{8.77}{3.01} & \pmv{\best{10.45}}{5.42} & \pmv{7.07}{2.17} \\
letter & \pmv{9.06}{0.00} & \pmv{\secondbest{9.19}}{0.00} & \pmv{\best{9.25}}{0.00} & \pmv{9.10}{0.00} & \pmv{9.04}{0.00} & \pmv{9.08}{0.00} \\
vertebral & \pmv{51.54}{6.33} & \pmv{52.63}{4.30} & \pmv{\secondbest{56.97}}{6.83} & \pmv{\best{58.91}}{5.66} & \pmv{51.50}{8.45} & \pmv{47.63}{5.89} \\
satimage-2 & \pmv{88.35}{0.00} & \pmv{\best{90.05}}{0.00} & \pmv{\secondbest{89.58}}{0.00} & \pmv{87.33}{0.00} & \pmv{82.36}{0.00} & \pmv{86.39}{0.00} \\
wdbc & \pmv{\best{82.54}}{4.97} & \pmv{74.76}{9.40} & \pmv{77.10}{5.60} & \pmv{\secondbest{78.85}}{5.57} & \pmv{70.67}{8.13} & \pmv{75.09}{6.50} \\
celeba & \pmv{\secondbest{8.05}}{0.55} & \pmv{\best{8.85}}{0.64} & \pmv{6.66}{0.57} & \pmv{7.61}{0.59} & \pmv{7.98}{0.73} & \pmv{6.42}{0.51} \\
landsat & \pmv{38.72}{0.00} & \pmv{\best{43.29}}{0.00} & \pmv{37.26}{0.00} & \pmv{33.92}{0.00} & \pmv{35.48}{0.00} & \pmv{\secondbest{41.83}}{0.00} \\
vowels & \pmv{42.19}{0.00} & \pmv{42.20}{0.00} & \pmv{36.20}{0.00} & \pmv{41.82}{0.00} & \pmv{\secondbest{43.53}}{0.00} & \pmv{\best{46.26}}{0.00} \\
shuttle & \pmv{98.60}{0.46} & \pmv{98.53}{0.83} & \pmv{98.23}{0.59} & \pmv{\best{99.11}}{0.53} & \pmv{\secondbest{98.92}}{0.55} & \pmv{98.34}{1.24} \\
wine & \pmv{99.38}{1.24} & \pmv{99.61}{0.79} & \pmv{99.51}{0.97} & \pmv{\secondbest{99.62}}{0.77} & \pmv{99.40}{1.20} & \pmv{\best{99.62}}{0.77} \\
annthyroid & \pmv{49.51}{0.00} & \pmv{\secondbest{50.76}}{0.00} & \pmv{47.56}{0.00} & \pmv{\best{51.84}}{0.00} & \pmv{47.08}{0.00} & \pmv{47.77}{0.00} \\
census & \pmv{15.63}{1.38} & \pmv{14.96}{1.04} & \pmv{15.78}{2.57} & \pmv{\secondbest{16.01}}{1.21} & \pmv{\best{16.05}}{1.61} & \pmv{15.61}{2.01} \\
waveform & \pmv{11.92}{0.00} & \pmv{12.19}{0.00} & \pmv{\secondbest{12.87}}{0.00} & \pmv{\best{13.29}}{0.00} & \pmv{12.10}{0.00} & \pmv{11.56}{0.00} \\
cardio & \pmv{\best{74.78}}{0.00} & \pmv{71.32}{0.00} & \pmv{67.28}{0.00} & \pmv{71.60}{0.00} & \pmv{66.73}{0.00} & \pmv{\secondbest{72.01}}{0.00} \\
wbc & \pmv{\secondbest{93.52}}{9.07} & \pmv{\best{95.80}}{5.20} & \pmv{91.12}{9.26} & \pmv{87.00}{11.82} & \pmv{86.24}{12.59} & \pmv{90.95}{9.97} \\
ionosphere & \pmv{98.60}{0.48} & \pmv{98.79}{0.42} & \pmv{98.71}{0.37} & \pmv{\best{98.90}}{0.33} & \pmv{98.76}{0.33} & \pmv{\secondbest{98.84}}{0.32} \\
cover & \pmv{84.09}{2.78} & \pmv{81.47}{4.19} & \pmv{83.37}{2.73} & \pmv{\best{86.87}}{2.47} & \pmv{85.65}{2.43} & \pmv{\secondbest{86.56}}{2.91} \\
pageblocks & \pmv{\best{61.90}}{0.00} & \pmv{58.90}{0.00} & \pmv{53.05}{0.00} & \pmv{\secondbest{59.84}}{0.00} & \pmv{48.22}{0.00} & \pmv{56.96}{0.00} \\
pendigits & \pmv{76.32}{0.00} & \pmv{81.49}{0.00} & \pmv{\best{89.65}}{0.00} & \pmv{87.81}{0.00} & \pmv{87.87}{0.00} & \pmv{\secondbest{89.16}}{0.00} \\
satellite & \pmv{80.27}{0.00} & \pmv{\best{80.94}}{0.00} & \pmv{79.53}{0.00} & \pmv{79.09}{0.00} & \pmv{78.91}{0.00} & \pmv{\secondbest{80.91}}{0.00} \\
magic.gamma & \pmv{88.32}{0.15} & \pmv{89.09}{0.10} & \pmv{89.15}{0.15} & \pmv{\best{90.54}}{0.12} & \pmv{87.50}{0.20} & \pmv{\secondbest{89.21}}{0.09} \\
donors & \pmv{\best{94.22}}{1.05} & \pmv{\secondbest{93.84}}{0.85} & \pmv{90.01}{2.09} & \pmv{92.41}{1.20} & \pmv{88.46}{3.31} & \pmv{88.46}{1.80} \\
lymphography & \pmv{99.76}{0.49} & \pmv{\secondbest{99.80}}{0.41} & \pmv{\best{99.80}}{0.41} & \pmv{99.28}{0.96} & \pmv{99.28}{0.96} & \pmv{99.55}{0.55} \\
mammography & \pmv{37.81}{1.36} & \pmv{\secondbest{45.75}}{1.58} & \pmv{31.85}{1.20} & \pmv{\best{49.29}}{0.87} & \pmv{34.93}{1.41} & \pmv{39.61}{0.84} \\
\bottomrule 
    \end{tabular}}  \label{tab:appendix_fulladbench_aucpr}
  \end{table}